\documentclass[11.5pt]{article}
\usepackage{authblk}  

\usepackage{amsmath,verbatim}
\usepackage{amssymb}
\usepackage{amsthm}
\usepackage{graphicx}
\usepackage{caption}
\usepackage{subcaption} 
\usepackage{booktabs}
\usepackage{algorithmicx} 
\usepackage{algorithm}
\usepackage{algpseudocode}
\usepackage[top=1.4in, bottom=1.4in]{geometry}

\usepackage[colorlinks=true,linkcolor=blue,citecolor=blue,urlcolor=blue]{hyperref}

\algnewcommand\algorithmicinput{\textbf{Input:}}
\algnewcommand\algorithmicoutput{\textbf{Output:}}
\algnewcommand\Input{\item[\algorithmicinput]}%
\algnewcommand\Output{\item[\algorithmicoutput]}%

\def\norminf#1{\|#1\|_{\infty}}
\def\abs#1{|#1|}

\def\eps{\varepsilon}
\def\R{\mathbb{R}}

\def\hh{\widehat{h}}

\newcommand{\uh}{u_\mathrm{high}}
\newcommand{\ul}{u_\mathrm{low}}
\newcommand{\ch}{c_\mathrm{high}}
\newcommand{\cl}{c_\mathrm{low}}
\newcommand{\cm}{c_\mathrm{mixed}}

\theoremstyle{thmstyletwo}%
\newtheorem{theorem}{Theorem}
\newtheorem{corollary}[theorem]{Corollary}
\newtheorem{remark}[theorem]{Remark}
\newtheorem{model}[theorem]{Model}
\newtheorem{lemma}[theorem]{Lemma}
\numberwithin{equation}{section}
\numberwithin{theorem}{section}
\numberwithin{figure}{section}
\numberwithin{table}{section}
\numberwithin{algorithm}{section}

\title{\textbf{Mixed precision accumulation for neural network inference guided by componentwise forward error analysis}}

\author[1]{El-Mehdi El Arar\thanks{Corresponding author: \href{mailto:mehdi.elarar@lip6.fr}{mehdi.elarar@lip6.fr}}}
\author[2]{Silviu-Ioan Filip}
\author[1]{Theo Mary}
\author[3]{Elisa Riccietti}

\affil[1]{Sorbonne Universit\'e, CNRS, LIP6, 4 Place Jussieu, F-75005, Paris, France}
\affil[2]{Inria, IRISA, Universit\'e de Rennes, 263 Av. G\'en\'eral Leclerc, F-35000, Rennes, France}
\affil[3]{ENS de Lyon, CNRS, Inria, Universit\'e Claude Bernard Lyon 1 LIP, UMR 5668, 69342, Lyon cedex 07, France}

\date{}

\begin{document}

\maketitle

\begin{abstract}
This work proposes a mathematically founded mixed precision accumulation strategy for the inference of neural networks. 
Our strategy is based on a new componentwise forward error analysis that explains the propagation 
of errors in the forward pass of neural networks.  
Specifically, our analysis 
shows that the error in each component of the output of a linear layer is proportional to the condition number
of the inner product between the weights and the input, multiplied by the condition number of the activation function.
These condition numbers can vary widely
from one component to the other, thus creating a significant opportunity to introduce mixed precision:
each component should be accumulated in a precision inversely proportional to the product of these condition numbers.
We propose a numerical algorithm that exploits this observation:
it first computes all components in low precision, uses this output to estimate the condition numbers, 
and recomputes in higher precision only the components associated with large condition numbers.
We test our algorithm on various networks and datasets and confirm experimentally that it can significantly improve
the cost--accuracy tradeoff compared with uniform precision accumulation baselines.
\end{abstract}


\section{Introduction}
Modern applications in artificial intelligence require increasingly complex
models and thus increasing memory, time, and energy costs for storing and deploying large-scale deep
learning models with parameter counts ranging in the millions and billions. 
This is a limiting factor
both in the context of training and of inference. While the growing training costs
can be tackled by the power of modern computing resources, notably GPU accelerators,
the deployment of large-scale models leads to serious
limitations in inference contexts with limited resources, such as embedded systems
or applications that require real-time processing.   

In recent years, the use of low precision arithmetic has emerged as a
successful strategy to decrease these costs, motivated by the
development of specialized hardware for machine learning, such as
Google's TPUs \cite{jouppi2017datacenter}, NVIDIA tensor cores~\cite{nvid24b}, and others~\cite{talib2021}, which
provide fast mixed precision matrix multiply--accumulate (MMA) operations.
Low precision is usually introduced in trained neural networks through the quantization of weights and activations, that is, 
by storing the network parameters and intermediate signals in low precision~\cite{gholami2022survey}.
Indeed, inference compute workloads are dominated by MMA operations,
which can be accelerated by using lower precision arithmetic. 
Quantization therefore significantly reduces the inference cost,
usually in exchange for minor reductions in model accuracy. It has indeed been empirically shown
that neural network inference can be done effectively even when weights and
activations are stored using 8 bits~\cite{nagel2021white,gholami2022survey}.  

While weights and activations are commonly stored in low precision, the accumulation
is usually done in high precision. This is partly because most specialized MMA
hardware mentioned above provide the capability of accumulating in high
precision with little or no performance penalty~\cite{bhlm20}, and partly
because accumulating in low precision can create significant numerical issues,
ranging from overflow to excessive rounding error
accumulation~\cite{high:ASNA2}.
Nevertheless, 
reducing the accumulation precision
can be an effective strategy to increase performance
for general-purpose processors~\cite{de2020quantization,ni2021wrapnet,xie2021overflow}.
This motivates further research on how to reduce this precision
as much as possible while avoiding numerical issues and preserving the
model accuracy. This is the main goal of this work.

There exist several approaches to reduce the accumulation of errors in finite precision,
and many of them have been considered for improving the accuracy of the training and/or the inference workflows of deep neural networks.
For example, stochastic rounding~\cite{chm21,cfhm22, eddp22} prevents errors from
accumulating all in the same direction and thus improves the average accuracy;
it has been used to accelerate training~\cite{gupta2015deep, effm25}.
Blocked summation methods~\cite[Chap.~4]{high:ASNA2}, \cite{bhm20} 
reduce the worst-case error bounds by constraining the summation order, and
have also been used for training acceleration~\cite{Wang2018}.
Scaling techniques can help to avoid overflow and minimize underflow~\cite{mami25},
and have especially received focus in the context of fixed-point
arithmetic~\cite{sakr2019accumulation,xie2021overflow,ni2021wrapnet,colbert2023quantized,colbert2024accumulator}.

All previously mentioned approaches only consider \textit{uniform precision} accumulation, that is, 
the accumulation precision is the same across all inner products (multiply--accumulate operations).
In this work, we will instead focus on \textit{mixed precision} accumulation, that is,
we will allow different inner products to be performed in different precisions.
The main advantage of mixed precision approaches is that
they can leverage the possible differences in sensitivity 
of different parts of the computation: whereas a uniform precision scheme would be
limited by the most sensitive parts that require the highest precision, a
mixed precision scheme can adaptively keep only these parts in high precision,
while switching the less sensitive parts to lower precision---ideally without
(significantly) impacting the model accuracy.

While mixed precision approaches have been extensively investigated for
\textit{quantization}~\cite{lin2016fixed,dong2020hawq,dong2019hawq,yao2021hawq,gong2019mixed,uhlich2019mixed,wang2019haq,yang2021bsq},
to the best of our knowledge, they have not been previously considered for \textit{accumulation}.
This work is therefore completely complementary to existing studies. On the one hand, 
our approach is agnostic with respect
to the quantization method (that is, it applies to any network, regardless of how it has been quantized). On the other hand, it considers two different accumulation precisions with unit roundoffs $\ul$ and $\uh$,
but does not otherwise make any specific assumptions on how this accumulation is performed: that is, our approach may be
combined with stochastic rounding, blocked summation, etc.; the specific choice of accumulation method will simply determine just how low $\ul$ can be,
and how high $\uh$ needs to be.

The key question that our work addresses is: how should we decide which inner
products to perform in which precision? Our approach aims at answering this question
in a mathematically founded way by basing the precision choice criterion on a rigorous error
analysis. 
We develop such an
analysis that considers inexact inference using a multilayer perceptron architecture with a very generic error model.
Our analysis is in spirit quite similar to the recent work of Beuzeville et al.~\cite{bbgm__},
which also analyzes the inference of neural networks in presence of errors.
However, there are some key differences between the two analyses. 
Indeed, Beuzeville et al. perform a \textit{backward} error analysis,
whereas we will focus on the \textit{forward} error. There are advantages to
both types of analyses: backward error analysis yields bounds that are mostly
independent of the neural network parameters (they depend on the number and
size of the layers, but not on the actual values of the weights), and allows
for establishing the numerical stability of inference---the main goal and result of \cite{bbgm__}.
 In contrast, the goal of our forward error analysis is
completely different: we seek bounds that directly relate the errors incurred
in each inner product to the accuracy of the final output of the network, in
order to identify possible mixed precision opportunities; thus, our bounds
strongly depend on the network parameter values, and this is precisely what we
exploit to develop a mixed precision strategy.  Most importantly, the analysis
of Beuzeville et al. bounds the \textit{normwise} error, that is, the error is
only bounded in (some) norm; this does not allow to distinguish the impact of
errors incurred in different components of each layer of the network: the
errors across different components are ``smudged'' together in norm. In
contrast, our analysis bounds the \textit{componentwise} error;
this allows us to precisely identify the size of the errors in each component.
In particular, we make the key observation
that the error incurred in each component is proportional to both the condition number 
of the inner product and the condition number of the activation function evaluated at
that component. In order to balance the errors across all components, we should therefore 
set the precision of each inner product to be inversely proportional to the associated condition number. 
Because the magnitude of these condition numbers can vary widely 
from one component to the other, this creates a significant opportunity for mixed precision.

To summarize, the first main contribution of this work is to 
perform a componentwise forward error analysis that guides us towards
a mixed precision inference evaluation strategy.
The second main contribution of this work is to develop a practical mixed precision 
algorithm that is guided by this analysis.
In order to make the algorithm practical, we must introduce some approximations: computing the exact condition numbers
would indeed be too expensive. Motivated by some empirical observations,
we however show that the condition numbers can be cheaply estimated as a by-product of the output of each layer computed in low precision. 
Therefore, we propose the following approach, summarized in Figure~\ref{fig.schema}: at each layer $\ell$, we first compute
the output $h_\ell=\phi_\ell(W_\ell h_{\ell-1})$ entirely (uniformly) in a low precision $\ul$.
Then, we estimate the condition number $\kappa_\ell$ and check each of its components $(\kappa_\ell)_i$:
components for which the condition number is small enough ($(\kappa_\ell)_i \le \tau$, for some tolerance $\tau$) are kept in precision $\ul$, whereas
those for which the condition number is too large ($(\kappa_\ell)_i > \tau$) are recomputed using a higher precision $\uh$.

We test the proposed algorithm on multilayer perceptron networks of various depth, trained on
the MNIST and Fashion MNIST datasets. Our experiments show that the algorithm can achieve
a flexible cost--accuracy tradeoff, tunable via the tolerance parameter $\tau$.
Crucially, the achieved tradeoff is in many cases significantly better than with uniform precision accumulation:
that is, our mixed precision accumulation approach can significantly improve the model accuracy
compared with a uniform low precision approach, for a significantly lower cost than the uniform high precision approach.

The rest of the paper is organized as follows: in section~\ref{sec:analysis} we
carry out our error analysis and discuss its significance. In section~\ref{sec:algo},
we develop an inference algorithm with mixed precision accumulation.
We test the algorithm experimentally in section~\ref{sec:num}.
Finally, we conclude in section~\ref{sec:concl}.

\begin{figure}
\includegraphics[width=\textwidth]{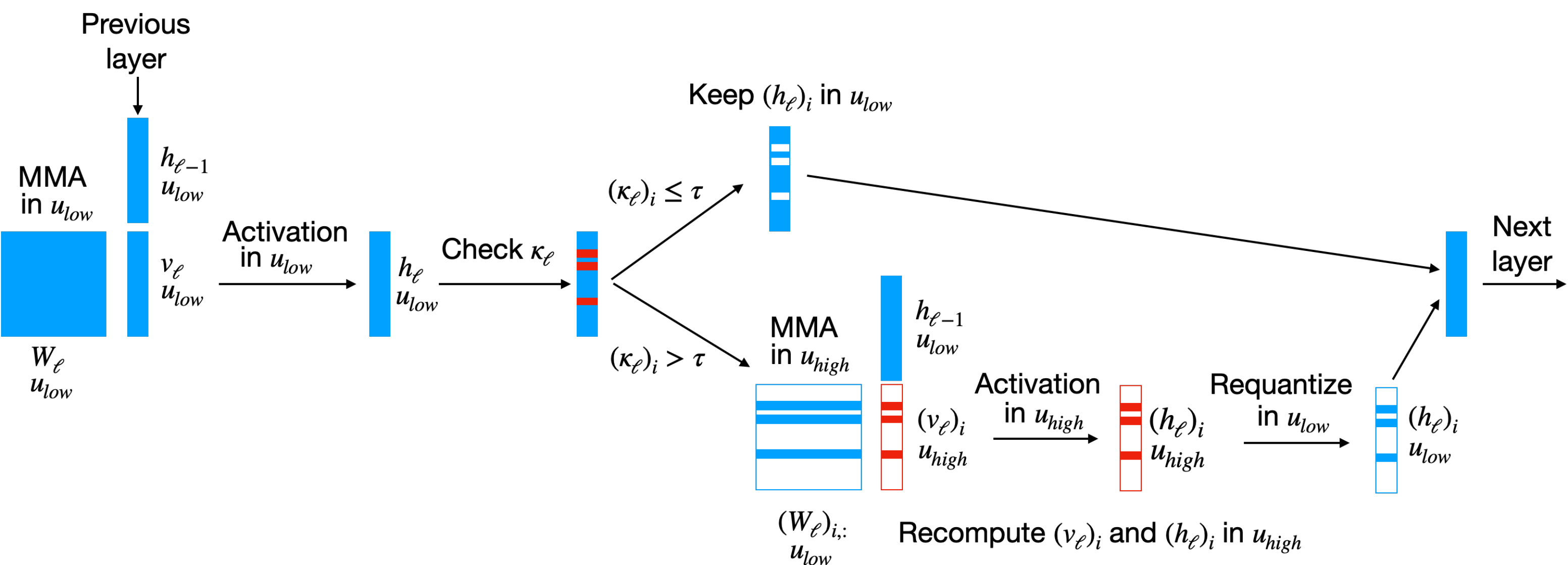}
\caption{Illustration of our inference approach 
with mixed precision accumulation (Algorithm~\ref{alg:ALG1}). At each layer $\ell$ we first compute 
the MMA $v_\ell = W_\ell h_{\ell-1}$ (where $h_{\ell-1}$ is the output of the previous layer)
and the activation $h_\ell = \phi_\ell(v_\ell)$ (where $\phi_\ell$ is the activation function)
in uniform low precision $\ul$ (blue). We estimate the condition number
$\kappa_\ell$ and use it to decide which components can be kept in low precision
(those for which $(\kappa_\ell)_i \le \tau$, for some tolerance $\tau$)
and which must be recomputed in higher precision $\uh$ (red); the latter are then requantized to low precision
and recombined with the components kept in low precision to produce the final output of the layer, which is passed to the next layer.
\label{fig.schema} }
\end{figure}

\section{Componentwise error analysis}\label{sec:analysis}

\subsection{Setting, notations, and error model}

We consider feedforward networks with $L$ layers, where each layer is indexed
by $\ell=1,\dots,L$ and composed of $n_\ell$ neurons.  We denote by
$W_\ell\in\R^{n_\ell\times n_{\ell-1}}$ the matrices of weights and by
$\phi_{\ell} : \R \mapsto \R$ the activation functions applied componentwise.
For an input $x \in \R^{n_0}$,
we denote $h_0 = x$ and for each layer $\ell$, the 
output of the layer $h_\ell\in\R^{n_{\ell}}$ is computed as
$$
h_\ell = \phi_\ell(W_\ell h_{\ell-1}).
$$
While we do use bias terms in the experiments in section~\ref{sec:num}, we do not include them
explicitly in the presented analysis for simplicity. The bias terms $b_\ell$ could be
easily included by redefining the 
weight matrices as $\overline{W}_\ell=\begin{bmatrix}
	{W}_\ell & {b}_\ell
\end{bmatrix}$ 
and the output of the $(\ell-1)$th layer as $\overline{h}_{\ell-1}=\begin{bmatrix} h_{\ell-1}& 1\end{bmatrix}^T$. We then have 
\[
{W}_\ell{h}_{\ell-1} + {b}_\ell = \begin{bmatrix}
	W_\ell & b_\ell
\end{bmatrix}\begin{bmatrix} h_{\ell-1} \\ 1\end{bmatrix}= \overline{W}_\ell \overline{h}_{\ell-1}.
\]

We will use the following notations.
Quantities affected by an error are marked by a hat.
We denote by $\circ$ the
Hadamard (componentwise) product and by $\oslash$ the Hadamard division; the Hadamard product of a matrix with a vector
multiplies the rows of the matrix by the components of the vector.
We denote by $\abs{\cdot}$ the absolute value, which
is applied componentwise for vectors and matrices. Inequalities between vectors $x\le y$ or matrices $A\le B$
of identical dimensions also apply componentwise; moreover, an inequality $A\le x$ between a matrix $A\in\R^{m\times n}$ and a vector $x\in\R^m$ 
applies to each row of $A$ componentwise, that is, $a_{ij} \le x_i$ for all $i,j$.
We denote by $\mathbf{1}$ the matrix or vector of all ones, and by $\norminf{\cdot}$ the infinity norm.

We seek to analyze the effect of errors in the computation of $h_\ell$. 
To do so, we will use the following generic error model. 
\begin{model}\label{model}
	We assume that $\hh_0 = h_0 = x$ and that each computed $\hh_\ell$ satisfies
	\begin{equation}\label{eq.hhl}
		\hh_\ell = \phi_\ell\big((W_\ell \circ(\mathbf{1} +\Delta W_\ell))\hh_{\ell-1}\big) \circ(\mathbf{1}+\Delta \phi_\ell),
		\quad \abs{\Delta W_\ell} \le \eps_\ell^W,
		\quad \abs{\Delta \phi_\ell} \le \eps_\ell^\phi,
	\end{equation}
	where $\Delta W_\ell\in\mathbb{R}^{n_\ell\times n_{\ell-1}}$,
	$\Delta\phi_\ell\in\mathbb{R}^{n_\ell}$, 
	$\eps_\ell^W\in\R^{n_\ell}$ is a nonnegative vector
	whose components bound the backward errors incurred in the evaluation of the
	matrix--vector product with $W_\ell$, so that $(\eps_\ell^W)_i = \max_{1 \le
		j \le n_{\ell-1}} \abs{(\Delta W_{\ell})_{ij}} $ for
	$i=1,\dots,n_\ell$, and $ \eps_\ell^\phi  \in \R^{n_\ell}$  is a nonnegative
	vector whose components bound the forward errors incurred in the evaluation of
	$\phi_\ell$. 
\end{model}

Note that Model~\ref{model}, while quite generic, is also very realistic.
Indeed, 
according to the IEEE 754 standard, 
in a floating-point arithmetic with unit roundoff $u$,
the matrix--vector products satisfy a relative backward error bound with
$\eps_\ell^W$ bounded by $n_{\ell-1}u$~\cite{high:ASNA2}. 
As for the activation functions, while the standard
does not prescribe how accurately mathematical functions should be evaluated,
in practice, most implementations satisfy a relative forward error bound of order a small multiple of $u$; see, for example,
 Appendix E of the CUDA programming guide\footnote{\url{https://docs.nvidia.com/cuda/archive/10.1/pdf/CUDA_C_Programming_Guide.pdf}}.

\subsection{Preliminaries}

We will need the following two inequalities on perturbed matrix--vector products.

\begin{lemma}\label{lem.DxA}
	Let $A\in\R^{m\times n}$,  $x\in\R^n$, and $\Delta x \in \R^n$. We have
	\begin{equation}
		\abs{A}\abs{x\circ\Delta x} \le \norminf{\Delta x} \abs{A}\abs{x}.
	\end{equation}
\end{lemma}

\begin{proof}
	Since the inequality is componentwise, it suffices to prove it for an arbitrary index $i$, $1 \le i \le m$. 
	The $i$th component of $ \abs{A}\abs{x\circ\Delta x} $ satisfies
	\[
	(\abs{A}\abs{x\circ\Delta x})_i = \sum_{j=1}^{n} \abs{a_{ij} x_j \Delta x_j} 
	\le  \norminf{\Delta x} \sum_{j=1}^{n} \abs{a_{ij}} \abs{x_j} = \norminf{\Delta x} (\abs{A}\abs{x})_i. \qedhere
	\]
\end{proof}

\begin{lemma}\label{lem.DA}
	Let $A\in\R^{m\times n}$, $x\in\R^n$, and $\Delta A \in\R^{m\times n}$ such that $\abs{\Delta A} \le \eps^A  \in \R^m$ with  $(\eps^A)_i  = \max_{1\le j \le n} \abs{\Delta a_{ij}}$. We have 
	\begin{equation}
		\abs{A \circ\Delta A}\abs{x} \le (\abs{A}\abs{x}) \circ \eps^A.
	\end{equation}
	
\end{lemma}

\begin{proof}
	Once again, since the inequality is componentwise, it suffices to prove it for an arbitrary index $i$, $1 \le i \le m$. 
	The $i$th component of $\abs{A \circ\Delta A}\abs{x}$ satisfies
	\[
	(\abs{A \circ\Delta A}\abs{x})_i = 
	\sum_{j=1}^{n} \abs{a_{ij} \Delta a_{ij} x_j}  
	\le  \sum_{j=1}^{n} \abs{a_{ij}} \abs{x_j} (\eps^A)_i 
	= ((\abs{A}\abs{x}) \circ \eps^A)_i. \qedhere
	\]
\end{proof}

Lemma~\ref{lem.DxA} states that multiplying a nonnegative matrix $\abs{A}$ with
a nonnegative vector $\abs{x}$ perturbed componentwise by $\abs{\Delta x}$
yields a result $\abs{A}\abs{x}$ whose $i$th component is perturbed by the largest of the components
of $\abs{\Delta x}$. 
Lemma~\ref{lem.DA} shows that a similar result holds when multiplying a nonnegative matrix $\abs{A}$ perturbed componentwise by
$\abs{\Delta A}$ with
a nonnegative vector $\abs{x}$: this 
yields a result $\abs{A}\abs{x}$ whose $i$th component is perturbed by the largest of the components
of the $i$th row of $\abs{\Delta A}$, $(\eps^A)_i  = \max_{1\le j \le n} \abs{\Delta a_{ij}}$.
In other words, perturbed matrix--vectors (with a componentwise perturbation on either the matrix or the vector)
contaminate the result by spreading the perturbation across its components. 

We next define two key quantities that will appear in the analysis: the condition numbers
of a matrix--vector product and of a function.

\paragraph{\textbf{Condition number of a matrix--vector product.}}
For $A\in\R^{m\times n}$ and  $x \in \R^n$, we have
\begin{equation}\label{eq.kappaAx}
	\abs{A}\abs{x} = \kappa_{A,x} \circ \abs{Ax},
\end{equation}
where $\kappa_{A,x}\in\R^m$ is the vector whose $i$th component
\begin{equation}\label{eq.kappaAx-def}
	(\kappa_{A,x})_i =  \frac{(\abs{A}\abs{x})_i}{\abs{Ax}_i}
\end{equation}
is the condition number of the dot product between the $i$th row of $A$ and $x$,
which reflects the possibility of cancellation~\cite[sect.~1.7]{high:ASNA2} in
the computation of $Ax$ when $\abs{A}\abs{x} > \abs{Ax}$.

\paragraph{\textbf{Condition number of a function.}}
During the evaluation of $\phi_\ell(v)$ for some vector $v\in\R^{n_\ell}$,
we will also need to express a relative perturbation $\Delta v$ on the input $v$ as a relative perturbation $\Delta\phi_\ell(v)$
on the output $\phi_\ell(v)$. To do so, we introduce a function $\kappa_{\phi_\ell} : \R^{n_\ell} \mapsto \R_{+}^{n_\ell}$
that satisfies
\begin{equation}\label{eq.phi}
	\phi_\ell\big(v\circ(\mathbf{1}+\Delta v)\big) = \phi_\ell(v) \circ \big(\mathbf{1}+\kappa_{\phi_\ell}(v)\circ\Delta v'), \quad \Delta v'=\pm \Delta v.
\end{equation}
Equality \eqref{eq.phi} is stating that a relative perturbation $\Delta v$ on the input $v$
leads to a relative perturbation on the output $\phi_\ell(v)$ of magnitude
$\kappa_{\phi_\ell}(v)\abs{\Delta v}$ (note that we introduce a perturbation
$\Delta v'$ to account for a possible change of sign).

To obtain a more explicit expression of $\kappa_\phi$, consider the case where $v\in\R$. Then \eqref{eq.phi} becomes $\phi_\ell(v(1+\Delta v)) = \phi_\ell(v)(1+\kappa_{\phi_\ell}(v)\Delta v)$.
Assuming first that $\phi_\ell(v)\Delta v \ne 0$, this yields the expression
\begin{equation}
	\kappa_{\phi_\ell}(v) = \frac{\abs{\phi_\ell(v(1+\Delta v)) - \phi_\ell(v)}}{\abs{\phi_\ell(v)\Delta v}}.
\end{equation}
Taking the limit as $\Delta v$ goes to zero gives the condition number of $\phi_\ell$ at $v$, $|v\phi'_\ell(v)/\phi_\ell(v)|$ 
which shows that $\kappa_{\phi_\ell}$ can be interpreted as the condition number of $\phi_\ell$ to first order~\cite[sect.~1.6]{high:ASNA2}.
The case where $\phi_\ell(v)\Delta v=0$ requires special care.
If $\Delta v = 0$, or if
$\phi_\ell(v)=\phi_\ell(v(1 + \Delta v))=0$, then \eqref{eq.phi} is satisfied for
any $\kappa_{\phi_{\ell}}$, so we may in particular define $\kappa_{\phi_\ell} = 0$.
If $\phi_\ell(v)=0$ but $\phi_\ell(v(1 +  \Delta v)) \ne 0$, then there does not
exist any finite $\kappa_{\phi_\ell}$ such that \eqref{eq.phi} is satisfied, and so we define $\kappa_{\phi_\ell} = \infty$.
To summarize, we have the explicit expression of $\kappa_{\phi_\ell}$
\begin{equation}\label{eq.kappaphi}
	\kappa_{\phi_\ell}(v) = \left\{ \begin{array}{ll}
		\frac{\abs{\phi_\ell(v(1+\Delta v)) - \phi_\ell(v)}}{\abs{\phi_\ell(v)\Delta v}} & \text{ if } \phi_\ell(v)\Delta v\ne0                                \\
		0                                                                              & \text{ if } \phi_\ell(v)=\phi_\ell(v(1+\Delta v))                      \\
		\infty                                                                         & \text{ if } \phi_\ell(v)=0 \text{ and } \phi_\ell(v(1+\Delta v))\ne 0.
	\end{array} \right.
\end{equation}
Note that the fact that $\kappa_{\phi_\ell}$ can take $\infty$ as a value is
largely an artifact of considering relative perturbations. For example, for
ReLU activation, $\kappa_{\phi_\ell}(v) = \infty$ occurs only when 
we simultaneously have $v < 0$ and $v(1 + \Delta v) > 0$.
These conditions are met when $\Delta v < -1$, which corresponds to a
relative error $\abs{\Delta v}$ greater than 1. 

Going back to the general case where $\phi_\ell$ takes $v\in\R^{n_\ell}$ as input,
since \eqref{eq.phi} is a componentwise definition, we obtain the expression of the $i$th component of $\kappa_{\phi_\ell}(v)$
by applying \eqref{eq.kappaphi} to $\kappa_{\phi_\ell}(v_i)$.

\subsection{The analysis}

We are now ready to analyze the computation of $h_\ell$. We proceed by induction: assuming that the computed $\hh_{\ell-1}$ satisfies
\begin{equation}\label{eq.hhl-1}
	\hh_{\ell-1} = h_{\ell-1} \circ (\mathbf{1}+\Delta h_{\ell-1}), \quad \abs{\Delta h_{\ell-1}} \le \eps_{\ell-1}^h \in \R^{n_{\ell -1}}
\end{equation}
for some error term $\Delta h_{\ell-1}$ bounded componentwise by
$\eps_{\ell-1}^h$, we seek to determine $\Delta h_\ell$ and its corresponding
bound $\eps_\ell^h$.
Defining $v_\ell = W_\ell  h_{\ell-1}$ and injecting \eqref{eq.hhl-1} into \eqref{eq.hhl}, we obtain
\begin{align}
	\hh_\ell
	& = \phi_\ell\Big( \big(W_\ell  \circ (\mathbf{1} +\Delta W_\ell)\big) \big(h_{\ell-1}\circ(\mathbf{1}+\Delta h_{\ell-1})\big)\Big)
	\circ(\mathbf{1}+\Delta \phi_\ell) \nonumber                   \\
	& = \phi_\ell\Big( v_\ell +  \big(W_\ell \circ \Delta W_\ell \big) h_{\ell-1} 
	+ W_\ell \big( h_{\ell-1}\circ\Delta h_{\ell-1} \big) 
	+ \big(W_\ell \circ  \Delta W_\ell \big) \big(h_{\ell-1}\circ\Delta h_{\ell-1}\big) \Big) \circ(\mathbf{1}+\Delta \phi_\ell) \nonumber \\
	& = \phi_\ell\big(v_\ell \circ(\mathbf{1}+\Delta v_\ell)\big) \circ(\mathbf{1}+\Delta \phi_\ell), \label{eq.hhell}
\end{align}
with 
\begin{equation}
	\abs{\Delta v_\ell} \le 
	\Big(\abs{(W_\ell \circ \Delta W_\ell) h_{\ell-1}} 
	+ \abs{W_\ell ( h_{\ell-1} \circ \Delta h_{\ell-1} )}
	+ \abs{(W_\ell \circ \Delta W_\ell) (h_{\ell-1} \circ \Delta h_{\ell-1} )} \Big) \oslash \abs{v_\ell}.
\end{equation}
Using Lemmas~\ref{lem.DxA} and~\ref{lem.DA} together with~\eqref{eq.hhl} and~\eqref{eq.hhl-1}, we have
\begin{equation}
	\abs{\Delta v_\ell} \le (\abs{W_\ell} \abs{h_{\ell-1}}) \circ \eps_\ell^W \oslash \abs{v_\ell}
	+  \norminf{\eps^h_{\ell-1}} (\abs{W_\ell}\abs{h_{\ell-1}}) \oslash \abs{v_\ell} 
	+  \norminf{\eps^h_{\ell-1}} (\abs{W_\ell}\abs{h_{\ell-1}}) \circ \eps_\ell^W \oslash \abs{v_\ell}.
\end{equation}
By \eqref{eq.kappaAx} we have 
\begin{equation}\label{eq.kappav}
	(\abs{W_\ell}\abs{h_{\ell-1}}) \oslash \abs{v_\ell} = \kappa_{W_\ell, h_{\ell-1}} =: \kappa_{v_\ell},
\end{equation}
where, for the sake of readability, we abbreviate  $\kappa_{W_\ell, h_{\ell-1}}$ as $\kappa_{v_\ell}$.
We thus obtain
\begin{align}
	\abs{\Delta v_\ell} &\le \kappa_{v_\ell} \circ \big(\eps_\ell^W + \norminf{ \eps_{\ell-1}^h} \mathbf{1} +  \norminf{ \eps_{\ell-1}^h} \eps_\ell^W\big)
	\nonumber \\
	&=  \kappa_{v_\ell} \circ \big(\eps_\ell^W + \norminf{ \eps_{\ell-1}^h} ( \mathbf{1} + \eps_\ell^W) \big). \label{eq.Deltav}
\end{align}
Using \eqref{eq.phi} in \eqref{eq.hhell}, we have
\begin{align*}
	\hh_\ell & = \phi_\ell(v_\ell)\circ(\mathbf{1}+\kappa_{\phi_\ell}(v_\ell)\circ\pm\Delta v_\ell) \circ(\mathbf{1}+\Delta \phi_\ell)\nonumber                                                  \\
	& = h_\ell\circ(\mathbf{1}+\kappa_{\phi_\ell}(v_\ell)\circ\pm\Delta v_\ell) \circ(\mathbf{1}+\Delta \phi_\ell) \nonumber                                                          \\
	& = h_\ell\circ(\mathbf{1}+\kappa_{\phi_\ell}(v_\ell)\circ\pm\Delta v_\ell + \Delta \phi_\ell +\kappa_{\phi_\ell}(v_\ell)\circ\pm\Delta v_\ell \circ \Delta \phi_\ell)  \nonumber \\
	& = h_\ell\circ(\mathbf{1}+\Delta h_\ell)
\end{align*}
with
\begin{align}\label{eq.Deltah-0}
	\abs{\Delta h_\ell} 
	&\le \kappa_{\phi_{\ell}}(v_\ell) \circ \abs{\Delta v_\ell} 
	+ \abs{\Delta \phi_\ell} + \kappa_{\phi_{\ell}}(v_\ell) \circ \abs{\Delta v_\ell} \circ \abs{\Delta \phi_\ell} \nonumber \\
	&= \kappa_{\phi_{\ell}}(v_\ell) \circ \abs{\Delta v_\ell} \circ (\mathbf{1}+\abs{\Delta \phi_\ell}) + \abs{\Delta \phi_\ell} 
\end{align}
Combining \eqref{eq.hhl} and \eqref{eq.Deltav} into \eqref{eq.Deltah-0}, we finally obtain
\begin{equation}\label{eq.Deltah}
	\abs{\Delta h_\ell} 
	\le \kappa_{\phi_\ell}(v_\ell)\circ\kappa_{v_\ell}\circ\big(\eps_\ell^W +  \norminf{ \eps_{\ell-1}^h} (\mathbf{1} 
	+  \eps_\ell^W)  \big)\circ(\mathbf{1}+\eps_\ell^\phi) + \eps_\ell^\phi =: \eps_\ell^h.
\end{equation}

We summarize our analysis in the following theorem.

\begin{theorem}\label{thm.main}
	Let $h_\ell = \phi_\ell(W_\ell h_{\ell-1})$ be computed inexactly such that the
	computed $\hh_\ell$ satisfies Model~\ref{model} and assume that $\hh_{\ell-1}$ is computed
  with a relative error bounded by $\eps_{\ell-1}^h$ as defined in \eqref{eq.hhl-1}.
	Then, we have
	\[
	\hh_\ell = h_\ell \circ (\mathbf{1}+\Delta h_\ell), \quad
	\abs{\Delta h_\ell} \le \eps_\ell^h,
	\]
	where $\eps_\ell^h$ satisfies the recurrence relation
	\[
	\eps_\ell^h = \kappa_{\phi_\ell}(v_\ell)\circ\kappa_{v_\ell}\circ\big(\eps_\ell^W + \norminf{ \eps_{\ell-1}^h} ( \mathbf{1} +   \eps_\ell^W ) \big)\circ(\mathbf{1}+\eps_\ell^\phi) + \eps_\ell^\phi,
	\]
	where $\kappa_{\phi_\ell}$ satisties \eqref{eq.phi} and $\kappa_{v_\ell}$ is defined in \eqref{eq.kappav}.
\end{theorem}

\subsection{Interpretation of the analysis and consequences}

We now explain why this analysis reveals important features of the behavior of the forward propagation
under error perturbations, and motivates the use of mixed precision.
Theorem~\ref{thm.main} shows that, to first order, we have the recurrence
\begin{equation}\label{eq.recurr}
	\eps_\ell^h = \kappa_{\phi_\ell}(v_\ell) \circ \kappa_{v_\ell} \circ (\eps_\ell^W + \norminf{ \eps_{\ell-1}^h} \mathbf{1}) + \eps_\ell^\phi.
\end{equation}
This means that at layer $\ell$, the previously accumulated error
$\eps_{\ell-1}^h$ undergoes a series of transformations due to the propagation process. First, we add the local backward error $\eps_\ell^W$ accounting for the inexact matrix--vector product. Then, the combined error is scaled componentwise by the condition numbers $\kappa_{\phi_\ell}(v_\ell)$ and $\kappa_{v_\ell}$, which quantify the sensitivity of the layer's operations to input perturbations. This scaling reflects how the local structure of the layer amplifies the existing errors. Finally, we add the error $\eps_\ell^\phi$ accounting for the inexact evaluation of the activation function.

We can derive from recurrence \eqref{eq.recurr} a simpler scalar recurrence on $\norminf{\eps_\ell^h}$:
\begin{equation}
	\label{eq:eps_ell}
	\norminf{\eps_\ell^h} \le \norminf{\kappa_{\phi_\ell}(v_\ell) \circ \kappa_{v_\ell} \circ \eps_\ell^W}
	+ \norminf{\kappa_{\phi_\ell}(v_\ell) \circ \kappa_{v_\ell}} \norminf{\eps_{\ell-1}^h}
	+ \norminf{\eps_\ell^\phi}.
\end{equation}
This yields the following corollary.

\begin{corollary}\label{corollary}
	For all $\ell=1, \dots, L$,
	let
	\[
	\hh_\ell = h_\ell \circ (\mathbf{1}+\Delta h_\ell), \quad
	\abs{\Delta h_\ell} \le \eps_\ell^h,
	\]
	and assume $\eps_\ell^h$ satisfies the recurrence relation \eqref{eq:eps_ell}.
	Then the computed final output of the network, $\hh_L$, satisfies
	\[
	\hh_L = h_L \circ (\mathbf{1}+\Delta h_L), \quad
	\abs{\Delta h_L} \le \eps_L^h,
	\]
	with
	\begin{equation}\label{eq.recurr-simpl}
		\norminf{\eps_L^h} \le \sum_{\ell=1}^L\bigg[\Big(\prod_{k=\ell+1}^L \norminf{\kappa_{\phi_k}(v_k) \circ \kappa_{v_k}}\Big)
		\Big(\norminf{\kappa_{\phi_\ell}(v_\ell) \circ \kappa_{v_\ell} \circ \eps_\ell^W} + \norminf{\eps_\ell^\phi} \Big)\bigg].
	\end{equation}
\end{corollary}

\begin{proof}
	The proof is by induction on $L$. For $L=1$,  using \eqref{eq:eps_ell} gives
	\begin{align*}
		\norminf{\eps_1^h} \le \norminf{\kappa_{\phi_1}(v_1) \circ \kappa_{v_1} \circ \eps_1^W}
		+ \norminf{\kappa_{\phi_1}(v_1) \circ \kappa_{v_1}} \norminf{\eps_{0}^h}
		+ \norminf{\eps_1^\phi}.
	\end{align*}
	Since $\hh_0 = h_0$, $\eps_{0}^h$ is zero and~\eqref{eq.recurr-simpl} holds for $L=1$. 
	For the inductive step, assume that~\eqref{eq.recurr-simpl} is true for $L-1$. By \eqref{eq:eps_ell} we have
	\[
	\norminf{\eps_L^h} \le \norminf{\kappa_{\phi_L}(v_L) \circ \kappa_{v_L} \circ \eps_L^W}
	+ \norminf{\kappa_{\phi_L}(v_L) \circ \kappa_{v_L}} \norminf{\eps_{L-1}^h}
	+ \norminf{\eps_L^\phi} 
	\]
	and by the inductive assumption we thus obtain
	\begin{align*}
		\norminf{\eps_L^h} &\le  \norminf{\kappa_{\phi_L}(v_L) \circ \kappa_{v_L} \circ \eps_L^W}
		+ \norminf{\eps_L^\phi} \\
		&\qquad + \norminf{\kappa_{\phi_L}(v_L) \circ \kappa_{v_L}} \sum_{\ell=1}^{L-1}\bigg[\Big(\prod_{k=\ell+1}^{L-1} \norminf{\kappa_{\phi_k}(v_k) \circ \kappa_{v_k}}\Big)
		\Big(\norminf{\kappa_{\phi_\ell}(v_\ell) \circ \kappa_{v_\ell} \circ \eps_\ell^W} + \norminf{\eps_\ell^\phi} \Big)\bigg] \\
		&=  \norminf{\kappa_{\phi_L}(v_L) \circ \kappa_{v_L} \circ \eps_L^W}
		+ \norminf{\eps_L^\phi} \\
		&\qquad + \sum_{\ell=1}^{L-1}\bigg[\Big(\prod_{k=\ell+1}^{L} \norminf{\kappa_{\phi_k}(v_k) \circ \kappa_{v_k}}\Big)
		\Big(\norminf{\kappa_{\phi_\ell}(v_\ell) \circ \kappa_{v_\ell} \circ \eps_\ell^W} + \norminf{\eps_\ell^\phi} \Big)\bigg] \\
		&= \sum_{\ell=1}^L\bigg[\Big(\prod_{k=\ell+1}^L \norminf{\kappa_{\phi_k}(v_k) \circ \kappa_{v_k}}\Big)
		\Big(\norminf{\kappa_{\phi_\ell}(v_\ell) \circ \kappa_{v_\ell} \circ \eps_\ell^W} + \norminf{\eps_\ell^\phi} \Big)\bigg]. \qedhere
	\end{align*}
\end{proof}

Minimizing the error bound $\norminf{\eps_L^h}$ on the final output of the network thus amounts to minimizing
each of the error terms in sum \eqref{eq.recurr-simpl}. Assuming that the input $x$ and the weights of the network $W_\ell$ are fixed,
the only quantities under our control in this expression are $\eps_\ell^W$ and $\eps_\ell^\phi$, that is, the precision at which
we evaluate the matrix--vector products and the activation functions. We are interested in using the lowest possible precisions
while still achieving an error under a given accuracy target:
$\norminf{\eps_L^h} \le \varepsilon$. To do so, it seems sensible to equilibrate as much as possible the errors
on each of the terms in \eqref{eq.recurr-simpl}, that is, 
\begin{equation}\label{eq.precchoice}
	\Big(\prod_{k=\ell+1}^L \norminf{\kappa_{\phi_k}(v_k) \circ \kappa_{v_k}}\Big)
	\Big(\norminf{\kappa_{\phi_\ell}(v_\ell) \circ \kappa_{v_\ell} \circ \eps_\ell^W} + \norminf{\eps_\ell^\phi} \Big) \le \varepsilon/L.
\end{equation}

Equation
\eqref{eq.precchoice} shows that the errors incurred at layer $\ell$ are multiplied by the condition numbers of all the sucessive layers,
$\prod_{k=\ell+1}^L \norminf{\kappa_{\phi_k}(v_k) \circ \kappa_{v_k}}$. In principle, this quantity may vary across layers (in fact, 
it decreases monotonically as $\ell$ increases). However, because the errors are taken in infinity norm, only the maximum error components of
subsequent layers play a role: the potential variations across components are smudged together.
Since this term is moreover not easy to compute or estimate in practice, we decide to ignore it and rather consider the following criterion:
\begin{equation}\label{eq.precchoice2}
	\norminf{\kappa_{\phi_\ell}(v_\ell) \circ \kappa_{v_\ell} \circ \eps_\ell^W} + \norminf{\eps_\ell^\phi} \le \varepsilon/L.
\end{equation}

From this, we can immediately notice that the errors $\eps_\ell^\phi$ from the
activation functions appear in the infinity norm.
This suggests that it may be meaningless to
vary the precision of the activations between different components, because
only the maximum error component from the previous layer
is propagated; thus we may as well
compute all the components in the same precision. 

On the other hand, $\eps_\ell^W$ is multiplied componentwise by the condition number
\begin{equation}\label{kappal}
	\kappa_\ell := \kappa_{\phi_\ell}(v_\ell) \circ \kappa_{v_\ell},
\end{equation}
so we should try to balance each component of ${\kappa_{\ell}\circ\eps_\ell^W}$ to minimize their
maximum.
Therefore, we should choose the precision of the inner product with the $i$th row of $W_\ell$ to be inversely proportional
to the $i$th component of $\kappa_{\ell}$. This represents a good opportunity to introduce mixed precision in the forward pass: we expect the components of $\kappa_\ell$ to have a large dynamic range. Indeed, 
for typical activation functions such as ReLU or $\tanh$, Figure \ref{fig:cond} shows that
$\kappa_{\phi_\ell} \le 1$, and some of its components may be much smaller than
1, meaning that some inner products can be computed
in very low precision, while still maintaining a high accuracy on the overall computation.

In the next section we develop a mixed precision algorithm based on this reasoning.

\begin{figure}
	\centering
	\includegraphics[width=0.48\linewidth]{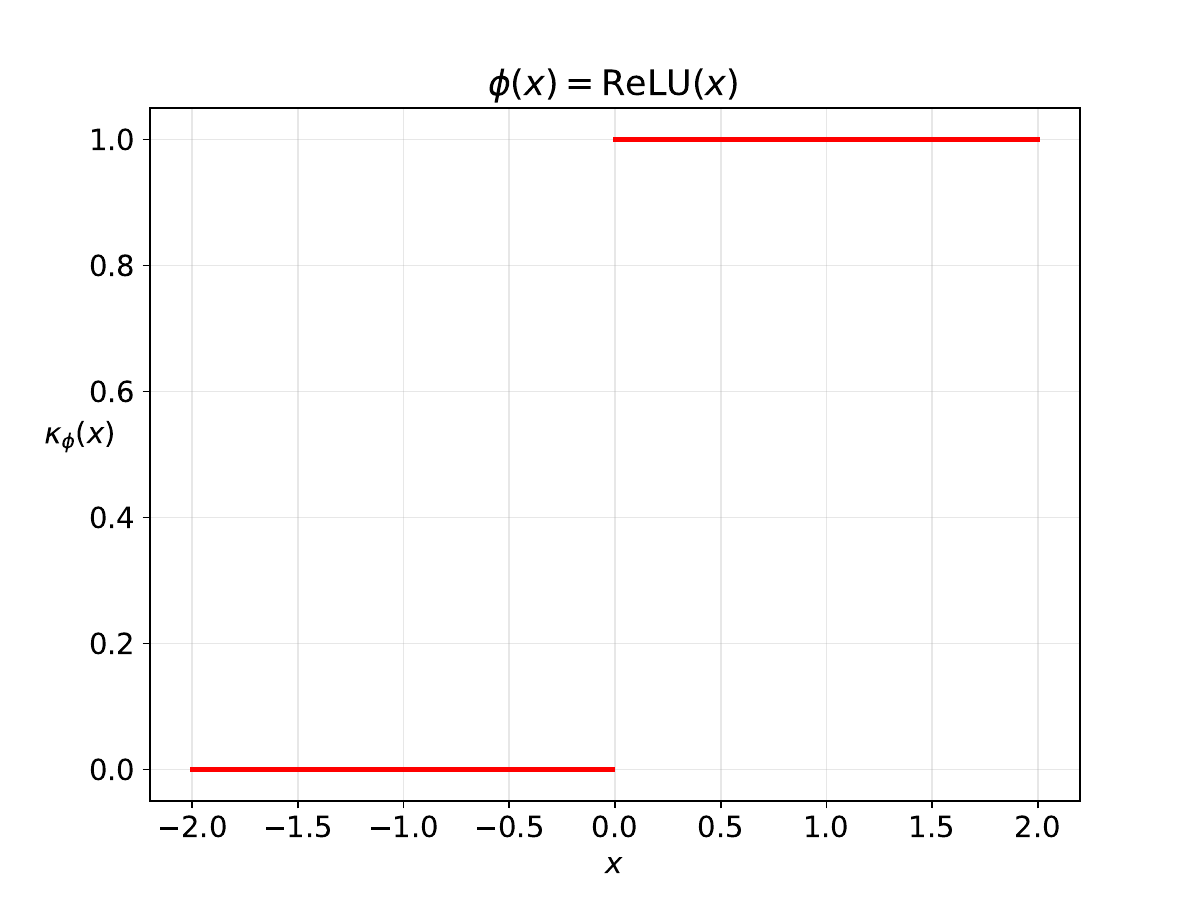}%
	\includegraphics[width=0.48\linewidth]{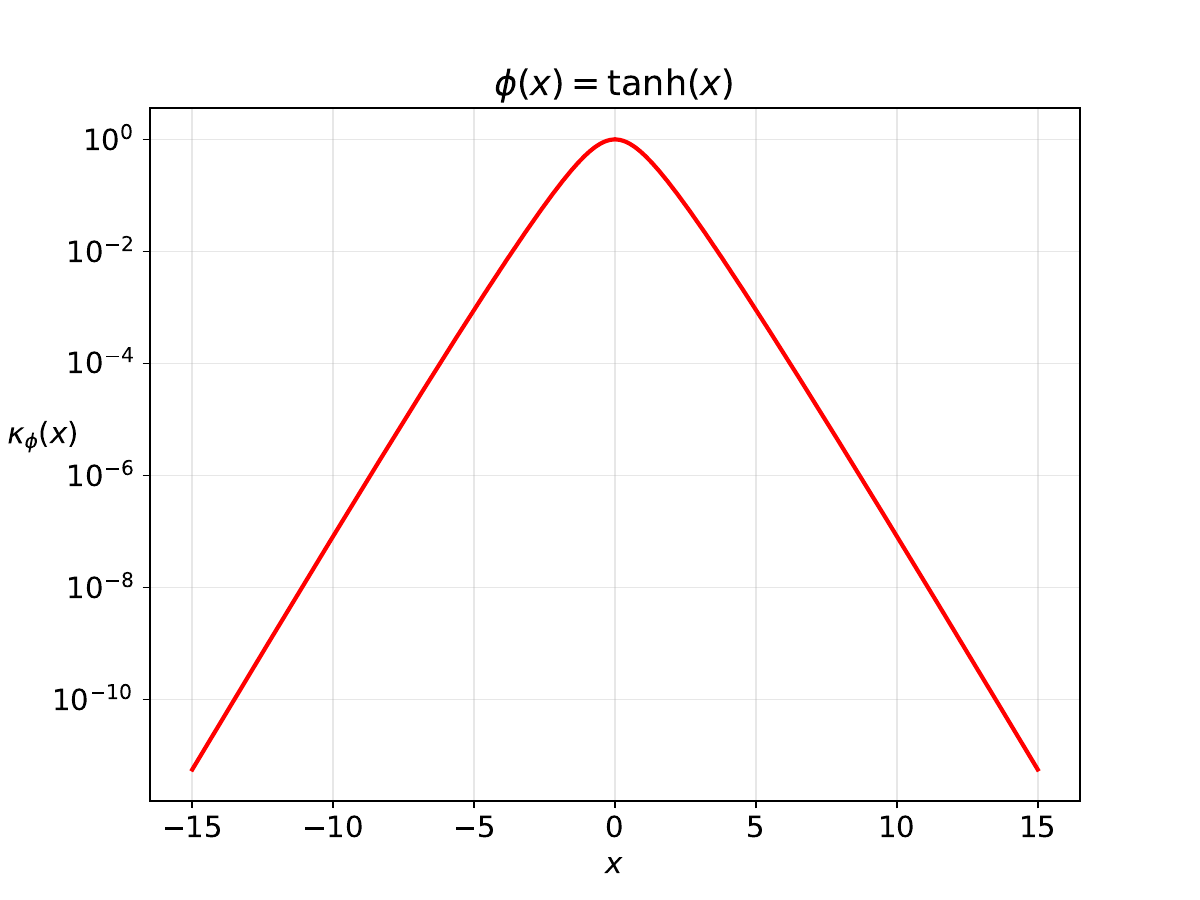}
	\caption{Condition number $\kappa_\phi(x)=\frac{|\phi'(x)x|}{|\phi(x)|}$ for $\phi(x)=\text{ReLU}(x)$ (left) and $\phi(x)=\tanh(x)$ (right). }
	\label{fig:cond}
\end{figure}

\section{A mixed precision algorithm for NN inference }
\label{sec:algo}

In this section, we show how to exploit the analysis presented in the previous
section to introduce mixed precision in the feedforward pass of neural
networks. We assume to have a trained network with given floating-point weights
$W_\ell$, $\ell=1,\dots,L$ stored in precision $\ul$, and we seek to exploit mixed precision in the
computation of the output of the network for a given input $x$.

\subsection{Main principle}

As discussed in the previous section, the errors at layer $\ell$ are
proportional to the product $\kappa_\ell \circ \eps_\ell^W$ (see \eqref{eq.precchoice2})
and our objective is to balance each component so as to minimize the maximum 
$\norminf{\kappa_\ell \circ \eps_\ell^W}$.
Ideally, if the condition numbers $\kappa_\ell$ were readily available, the precisions
of each inner product would simply be chosen such that $(\eps_\ell^W)_i \le \eps/(\kappa_\ell)_i$,
for a given target accuracy $\eps>0$; this choice would indeed yield 
$\norminf{\kappa_\ell\circ\eps_\ell^W} \le \eps$.
This shows that the precision used to compute 
each component of the $\ell$th layer should be chosen to be inversely proportional to the corresponding component
of the condition number $\kappa_\ell$. Let us consider the use of two precisions,
with unit roundoffs $\uh < \ul$. Then for large components of $\kappa_\ell$ we
should be careful in using the high precision $\uh$, whereas for small components,
the errors incurred will be damped and so we can safely use
the lower precision $\ul$ without impacting the accuracy of
the output. Concretely, we can introduce a tolerance $\tau>0$ which controls the precision switch criterion:
if $(\kappa_\ell)_i \le \tau$ we use precision $\ul$, otherwise we use precision $\uh$.
In particular, if the inner product between the $i$th row of $W_\ell$ and $h_{\ell-1}$
is implemented in floating-point arithmetic with a unit roundoff $u_i$ (equal to either $\ul$ or $\uh$),
rounding error analysis~\cite{jeru13} shows that $(\eps^W_\ell)_i = n_{\ell-1}u_i$.
Thus, in order for  $\norminf{\kappa_\ell\circ\eps_\ell^W} \le \eps$ to hold,
we should set the tolerance as $\tau = \eps/(n_{\ell-1} \ul)$.

\subsection{From a theoretical to a practical criterion: estimating $\kappa_\ell$}
\label{sec:simplif}

\begin{figure}
	\centering
	\includegraphics[width=0.5\linewidth]{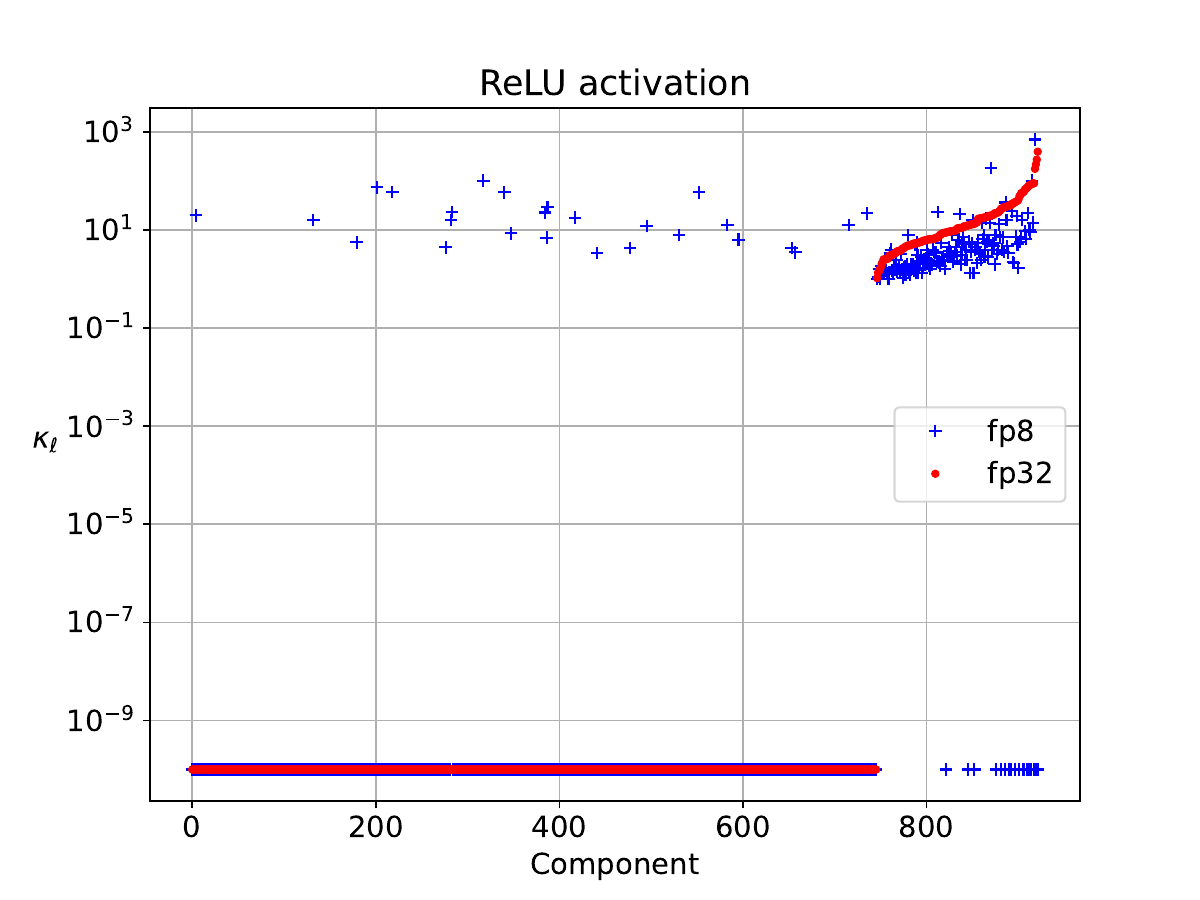}%
	\includegraphics[width=0.5\linewidth]{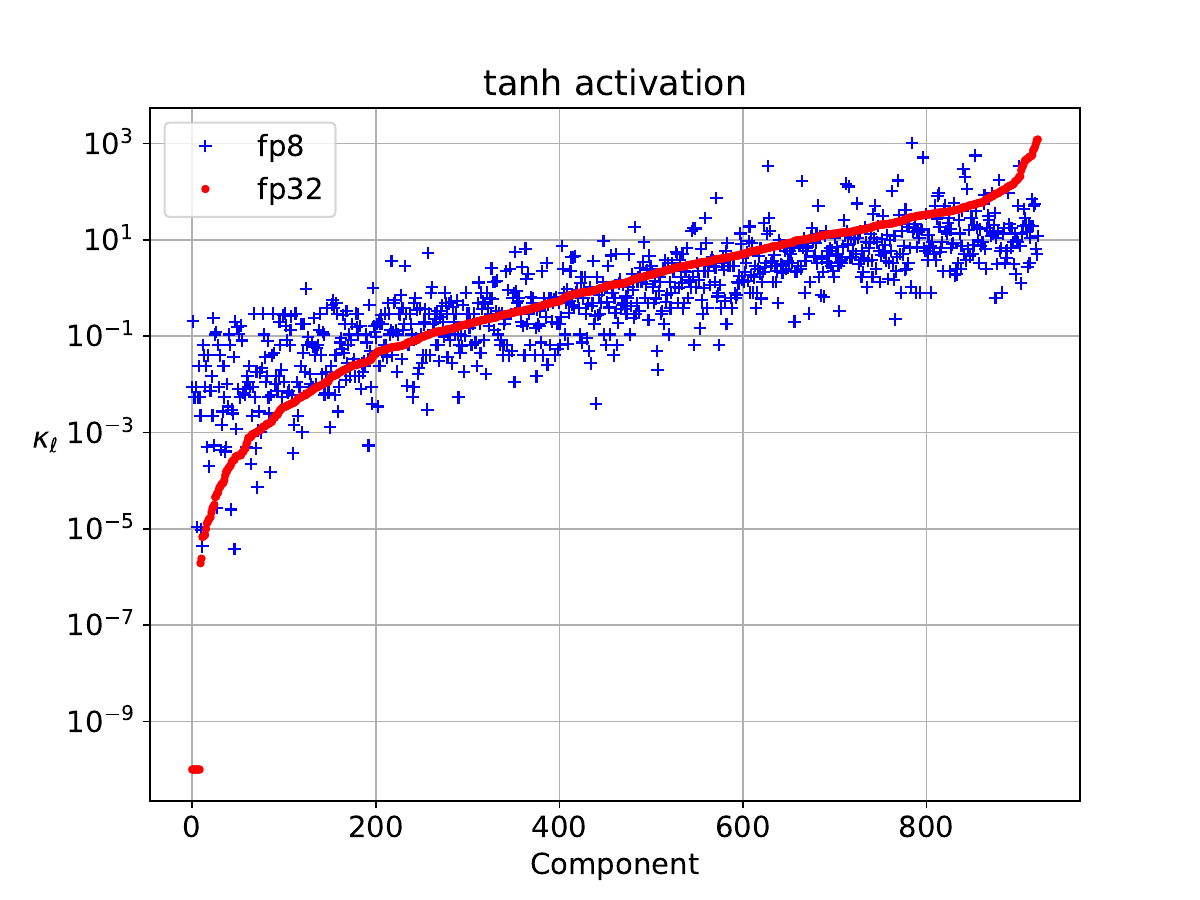}
	\caption{Comparison of the condition numbers $\kappa_\ell=\kappa_\phi\circ \kappa_{v_\ell}$ depending on whether they are computed in fp32 or in fp8,
		for a 3-layer network trained on the MNIST dataset with ReLU (left) and $\tanh$ (right) activations. 
		The values are sorted with respect to the fp32 condition numbers.}
	\label{fig:cond_fp8_fp32}
\end{figure}

\begin{figure}
	\centering
	\includegraphics[width=0.9\linewidth]{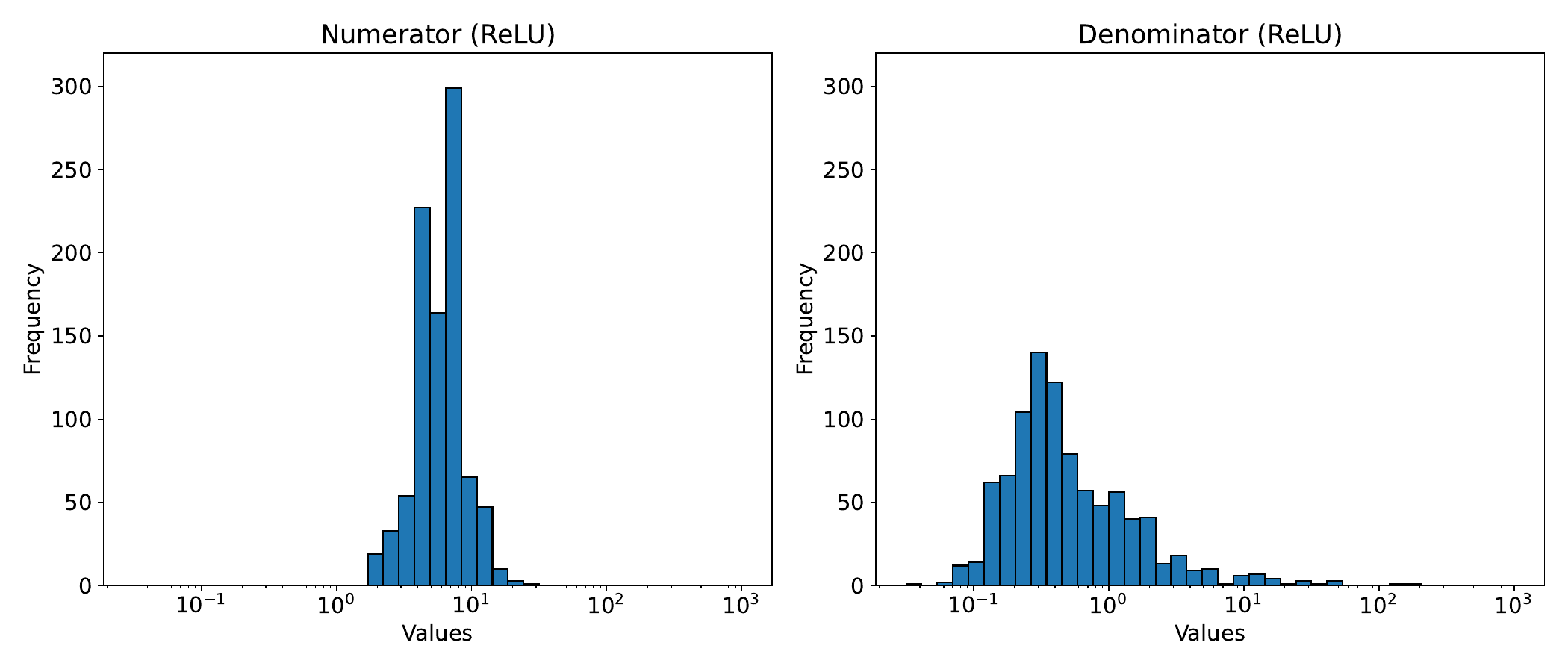}\\%
	\includegraphics[width=0.9\linewidth]{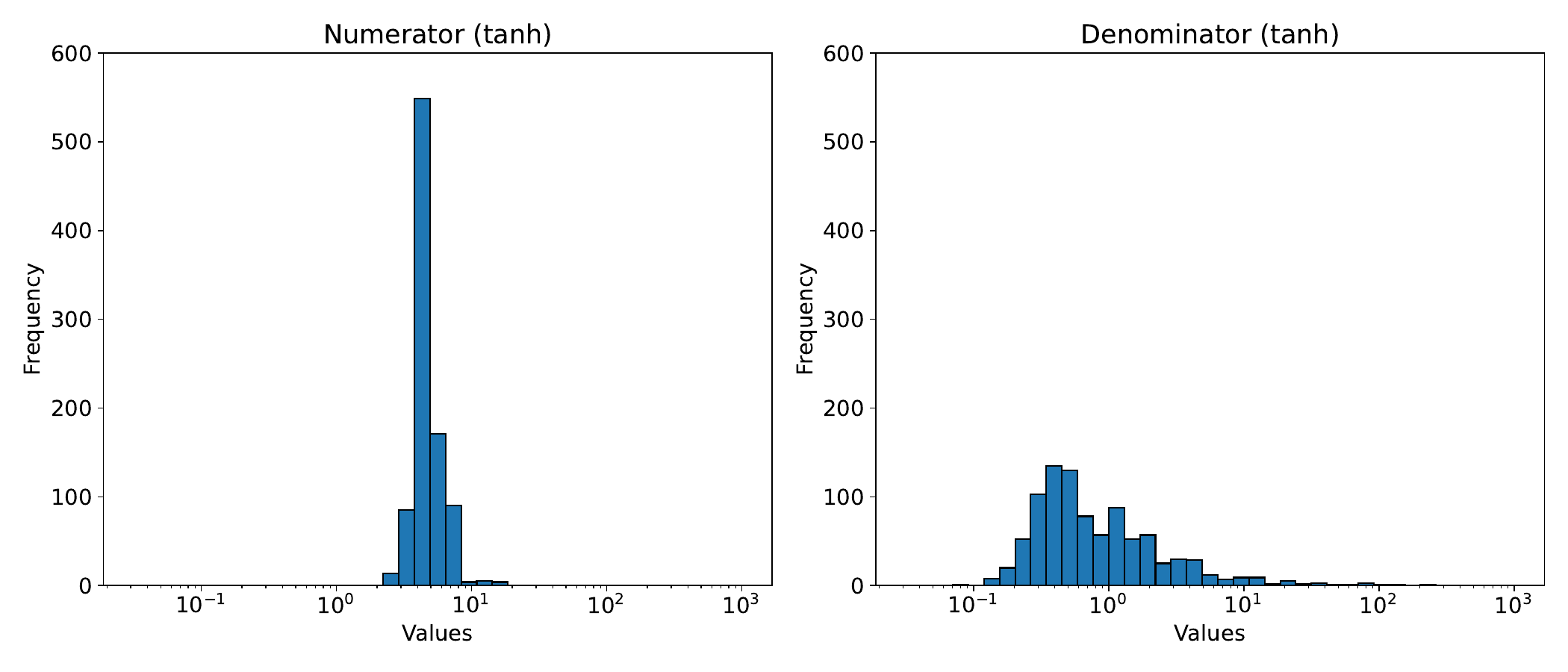}
	\caption{Distribution of the components of the numerator $\abs{W_\ell}\abs{h_{\ell-1}}$ (left) and of the denominator $\abs{W_\ell h_{\ell-1}}$ (right) of $\kappa_{v_\ell}$
		computed in fp8, for a three-layer network trained on the MNIST dataset with ReLU (top) and $\tanh$ (bottom) activations.}
	\label{fig:num_deno}
\end{figure}

While the principle behind this strategy would be mathematically ideal, 
unfortunately, since we do not know
the values of $\kappa_\ell$, it cannot be implemented as it is in practice. Indeed, it is worth recalling
that $\kappa_\ell = \kappa_{\phi_\ell}(v_\ell) \circ \kappa_{v_\ell}$ depends on $v_\ell = W_\ell h_{\ell-1}$;
therefore, computing $\kappa_\ell$ and thus $v_\ell$ in high precision would defeat the purpose of 
using mixed precision, since $v_\ell$ is precisely the result of the matrix--vector product that we aim to accelerate.
Moreover, for any layer $\ell$, $v_\ell$ depends in particular on $h_0 = x$, the input of the network,
so the precision choices depend on the input and cannot be reused across different inputs.

In order to obtain a practical
algorithm, we introduce some approximations. 
The key observation is that we do not need a very accurate computation of $\kappa_\ell$: estimating its order of magnitude 
is sufficient to decide which precision to use. Therefore, this suggests the following idea:
for each layer, compute first $v_\ell$ in precision $\ul$, that is, perform the entire matrix--vector product in low precision.
Then, use this approximate $v_\ell$ to compute an estimated $\kappa_\ell$ and check
the criterion for each component $(\kappa_\ell)_i$:
if $(\kappa_\ell)_i \le \tau$, the component $(v_\ell)_i$ computed in low precision can be kept,
whereas if $(\kappa_\ell)_i > \tau$, $(v_\ell)_i$ should be recomputed in high precision $\uh$.
This approach will therefore work best in situations where most components can be computed in low precision,
and high precision is only needed to recompute a few of the most sensitive components. Indeed, if the criterion leads to 
too many components needing to be recomputed, this mixed precision approach may end up being more expensive than simply computing
everything in high precision from the start.

To assess whether computing $\kappa_\ell$ in low precision 
is a reasonable approximation in practice,
we compare  in Figure \ref{fig:cond_fp8_fp32} the
values of the condition numbers computed in fp32 (red) with the corresponding values
computed in fp8 (blue). We use a three-layer perceptron network trained on the MNIST dataset for the ReLU
(left plot) and $\tanh$ (right plot) activation functions. The figure shows that the values
computed in fp8 follow the same trend as those computed in fp32, thus providing a reasonable
estimate of its order of magnitude. In particular, for the ReLU function, the vast majority of the zero values
(corresponding to negative components of $v_\ell$) in fp32 are correctly identified as zeros in fp8 also. 
There are a few outliers, in both directions: some fp32 zeros become nonzeros in fp8 (top left blue outliers), and may
be needlessly recomputed; conversely, some fp32 nonzeros become zeros in fp8 (bottom right blue outliers), and will be kept
in low precision even though they should be recomputed. These outliers represent a very small percentage of the components
and we may expect them not to have a significant impact on the inference accuracy.
Therefore, in the sequel, we use the low precision $\ul$ to compute the condition numbers $\kappa_\ell$.

Having computed a low precision $v_\ell$, estimating $\kappa_\phi(v_\ell)$ is straightforward: it suffices to compute
$\kappa_\phi(v_\ell) = \abs{v_\ell \circ \phi_\ell'(v_\ell) \oslash \phi_\ell(v_\ell)}$ in precision $\ul$.
Note that this formula involves computing $h_\ell = \phi_\ell(v_\ell)$ in precision $\ul$; the output $h_\ell$ of the $\ell$th layer 
in low precision is thus computed for free as part of this estimation; only the components of $h_\ell$ needing a higher precision will need
to be recomputed. Unfortunately, estimating 
$\kappa_{v_\ell}= (\abs{W_\ell}\abs{h_{\ell-1}})\oslash \abs{W_\ell h_{\ell-1}}$
is still too expensive, because of the expensive computation required by the numerator.
Indeed, computing this numerator for all
$\ell$ would cost the same as a full forward pass, since we need to
compute the matrix--vector products $\abs{W_\ell}\abs{h_{\ell-1}}$ for all
layers. We can however avoid this computation, thanks to the key observation that the
variations in magnitude of $\kappa_{v_\ell}$ are mostly due to variations of
the denominator. In particular, large condition numbers are typically due to
cancellation happening in the denominator, but not in the numerator. A common example is if the matrix is (approximately) Gaussian,
in which case the numerator is tightly concentrated around $n_{\ell-1}/2$ with high probability, whereas the denominator can vary between zero and a multiple of
$\sqrt{n_{\ell-1}}$, and can be arbitrarily close to zero with significant probability~\cite{hima20}.
We illustrate this in Figure~\ref{fig:num_deno}, which reports the distribution of the numerator and the
denominator in $\kappa_{v_\ell}$ for a three-layer network trained on MNIST. 
For both the ReLU (top) and the $\tanh$ (bottom) functions, the denominator (right) has
a much larger dynamic range than the numerator (left). As a consequence, it
seems reasonable to approximate the numerator by a  fixed constant $c$.

Figure~\ref{fig:kappa_new} confirms that the approximation 
$\kappa_{v_\ell} \approx \kappa_{v_\ell}' := c \mathbf{1} \oslash \abs{v_{\ell}}$ 
is reasonable. 
Note that there is no need to tune this constant $c$,
because it can be directly integrated in the criterion based on $\tau$: 
checking whether 
$c \kappa_{\phi_\ell}(v_\ell) \oslash \abs{W_\ell h_{\ell-1}} \le \tau$
is equivalent to checking whether 
$\kappa_{\phi_\ell}(v_\ell) \oslash \abs{W_\ell h_{\ell-1}} \le \tau'$
with $\tau' = \tau/c$.
Thus the tolerance $\tau$ is the only hyperparameter that needs tuning.

\begin{figure}[tb]
\centering
\includegraphics[width=0.5\linewidth]{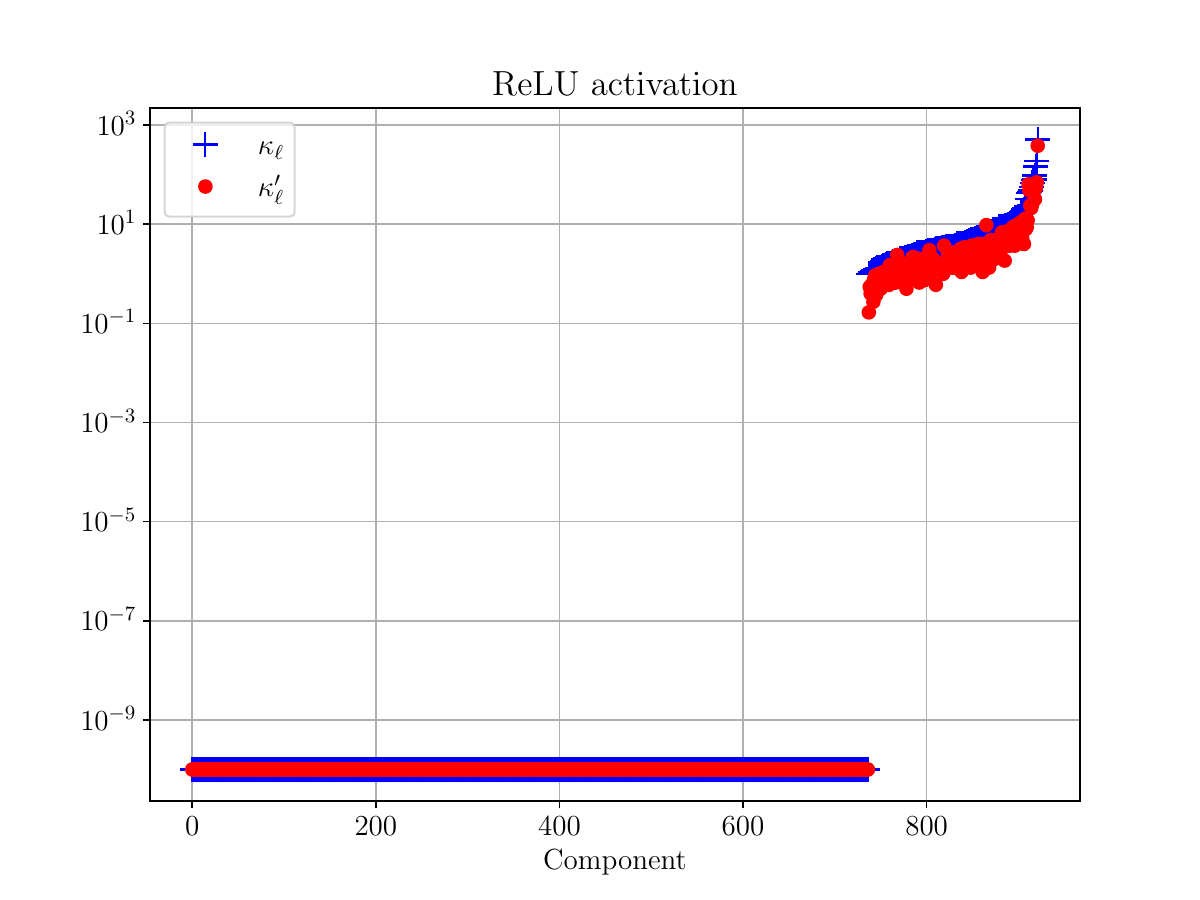}%
\includegraphics[width=0.5\linewidth]{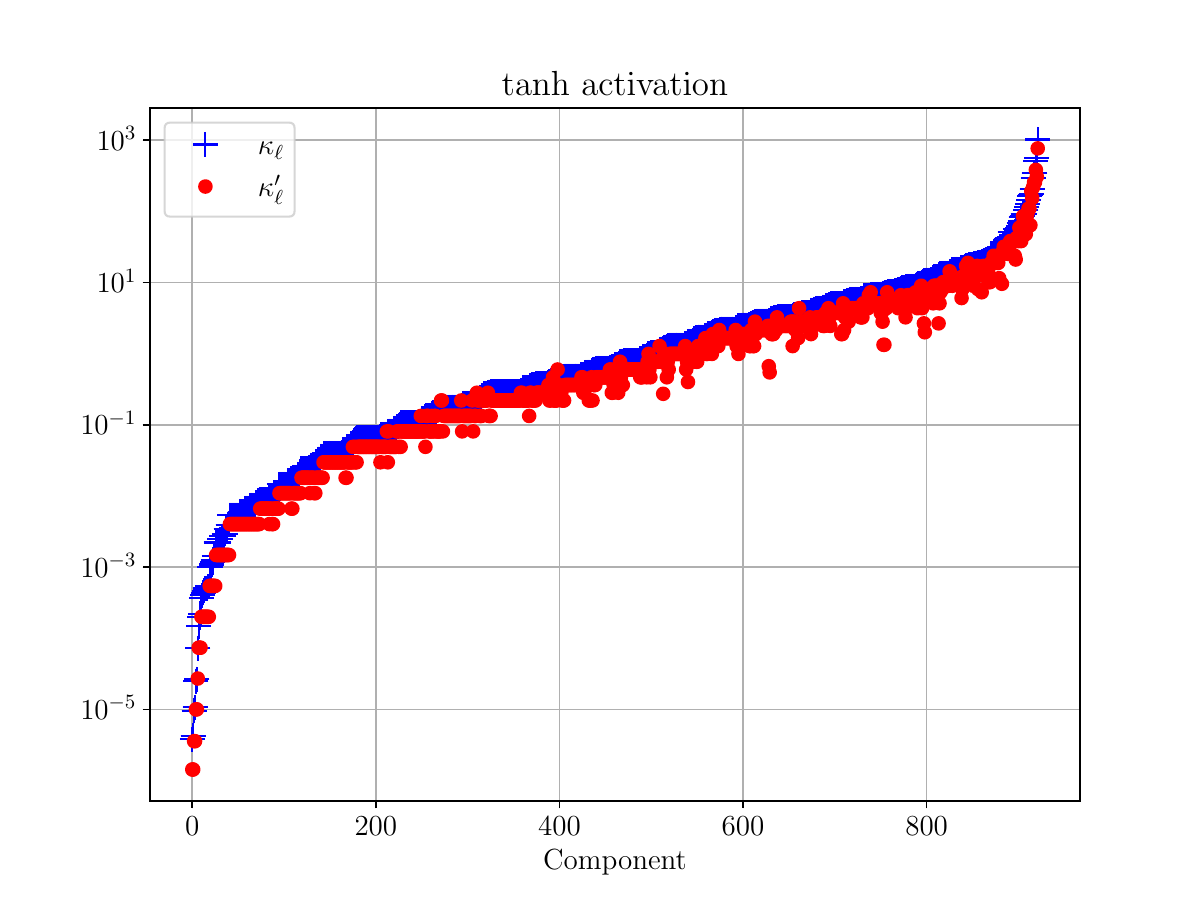}
\caption{Comparison of the condition number $\kappa_\ell=\kappa_\phi(v_\ell)\circ \kappa_{v_\ell}$ and its proposed
	approximation $\kappa_\ell'=\kappa_\phi\circ \frac{c}{\abs{W_\ell h_{\ell-1}}}$ (with $c=3$), 
	both computed in fp8, for a three-layer network trained on the MNIST dataset with ReLU (left) and $\tanh$ (right) activations.}

\label{fig:kappa_new}
\end{figure}

\subsection{The algorithm}

The successive approximations introduced above lead to a practical criterion for a mixed precision inference evaluation strategy. We summarize 
the proposed approach in Algorithm \ref{alg:ALG1}.

\begin{algorithm}[htbp]
	\caption{Neural network inference with mixed precision accumulation}
	\label{alg:ALG1}
	\begin{algorithmic}[1]
		\Input $W_1,\dots,W_L$, the weight matrices; $h_{0}=x$, the input vector; $\tau$, a tolerance controlling the precision choice; $\ul,\uh$, the precisions. 
		\Output $h_L$, the output of the network.
		\For{$\ell=1,\dots,L$}
		\State \label{step1} Compute $v_\ell= W_\ell  h_{\ell-1}$ in precision $\ul$.
		\State Compute $h_\ell = \phi_\ell(v_\ell)$ in precision $\ul$.
		\State Compute $\kappa_{\phi_\ell}(v_\ell) = \abs{v_\ell \circ \phi_\ell'(v_\ell)} \oslash \abs{\phi_\ell(v_\ell)}$ in precision $\ul$.
		\State \label{step2} Compute $\kappa_\ell = \kappa_{\phi_\ell} \oslash \abs{v_\ell}$ in precision $\ul$.
		
		\For{every component $(\kappa_\ell)_i$ }
		\If{$(\kappa_\ell)_i > \tau $}
		\State \label{step3} Recompute $(v_\ell)_i= (W_\ell  h_{\ell-1})_i$ in precision $\uh$.
		\State \label{step4} Recompute $(h_\ell)_i = \phi_\ell((v_\ell)_i)$ in precision $\uh$.
		\State \label{step-requant} Requantize $(h_\ell)_i$ back to precision $\ul$. 
		\EndIf
		\EndFor
		\EndFor
	\end{algorithmic}
\end{algorithm}

As mentioned previously, in order for the algorithm to be efficient,
the percentage of components that need to be recomputed in high precision must be small.
We now quantify this statement more precisely by using the following cost model.
We only consider the cost of the matrix--vector products $W_\ell h_{\ell-1}$. These require 
$O(n_\ell n_{\ell-1})$ floating-point operations, whereas the remaining steps of the algorithm (which essentially consist of the evaluation of the
activation functions and the estimation of the condition numbers) only require $O(n_\ell)$ operations/function evaluations.
Therefore for large-scale networks we may reasonably assume that the cost of the matrix--vector products will dominate---note that this specifically assumes
multilayer perceptron networks; see section~\ref{sec:concl} for a discussion on the extension to convolutional networks.

Let us thus focus on the matrix--vector products. 
Let $\cl$ be the cost of performing all the matrix--vector products (across all layers) in uniform precision $\ul$, and
let $\ch$ be the corresponding cost when using uniform precision $\uh$ instead.
Let $\rho \in [0,1]$ be the
fraction of components---and thus of inner products---that need to be
recomputed in precision $\uh$. Then the cost of the mixed precision Algorithm~\ref{alg:ALG1}
is 
\begin{equation}
\label{eq:mp_cost}
\cm = \cl + \rho \ch =  \left(\frac{\cl}{\ch}+\rho\right)  \ch.
\end{equation}
It is thus important to note that while we naturally have $\cl \le \cm$,
we cannot guarantee in general that $\cm \le \ch$: for this to hold,
we must have the condition $\cl/\ch + \rho < 1$.
In other words, 
the mixed precision cost will be less than the high precision one if
the costs ratio between the low and high precision is sufficiently small, and
the fraction of components that need to be recomputed in high precision is also sufficiently small.

\begin{remark}
Algorithm \ref{alg:ALG1} can
easily be extended to use more than two precisions. Indeed, given a list of precisions with unit roundoffs 
$u_1 > \ldots > u_p$,
we can first compute $h_\ell$ in precision $u_1$ and check the components of $\kappa_\ell$
against a list of tolerances $\tau_1 < \ldots < \tau_{p-1} < \tau_p := \infty$.
Components $(\kappa_\ell)_i \in (\tau_{j},\tau_{j+1}]$ are then recomputed in precision $u_j$,
for $j=1\colon p-1$.
The cost model \eqref{eq:mp_cost} then becomes
\[
\cm = c_1 + \sum_{j=2}^{p} \rho_j c_j,
\]
where $c_j$ is the cost of computing an MMA in precision $u_j$ and $\rho_j$ is the fraction of
components that are recomputed in precision $u_j$.
\end{remark}

\begin{remark}
Line~\ref{step-requant} of Algorithm \ref{alg:ALG1} assumes the weights are quantized in precision $\ul$. However, in principle, the quantization
precision can be chosen independently of the two accumulation precisions; the quantized weights can even be in mixed precision. 
Note that Algorithm~\ref{alg:ALG1} does not require a copy
of the weights in different precisions; only the quantized weights are needed.
\end{remark}

\section{Numerical experiments}\label{sec:num}

In this section we experimentally assess the potential of the mixed precision strategy introduced
in Algorithm \ref{alg:ALG1}.

\paragraph{\textbf{Experimental setting and description of the figures.}}

We consider multilayer perceptron networks~\cite{mlp} with 3, 5, or 8 layers (including both the hidden and input/output layers),
with either ReLU or $\tanh$
activation functions. The weight matrices for an $L$-layer network have
dimensions $784\times 784$ for the first $L-2$ layers, $128\times 784$ for
layer $L-1$, and $10\times 128$ for layer $L$.

Our experiments use floating-point arithmetic, with two different formats:
the fp8-E4M3 format~\cite{ocp23a},
an 8-bit format with 4 bits dedicated to the
exponent and 3 bits to the mantissa, 
and the IEEE 754 fp16 format~\cite{ieee19}, a 16-bit format with 5
bits dedicated to the exponent and 10 bits to the mantissa. 
Hereinafter, we denote these two formats simply as fp8 and fp16, respectively.
We leverage the \texttt{mptorch}~\cite{mptorch} Python library to faithfully simulate reduced precision computations in fp8. 

For all experiments, the neural networks considered are pre-trained on the MNIST and Fashion MNIST datasets using IEEE 754 fp32 (single precision) arithmetic and a quantization-aware training approach~\cite[sect.~4]{nagel2021white} where the weights are quantized to the target fp8 format. 

We consider and compare three accumulation strategies in performing feed-forward computation on the chosen networks:
two uniform precision variants, which use the same accumulation precision (either fp8 or fp16) across all components, and our mixed precision variant (Algorithm~\ref{alg:ALG1}), which uses
fp8 as the low precision $\ul$ and fp16
as the high precision $\uh$.

On most hardware, we can expect fp8 arithmetic to be twice as fast as fp16 arithmetic. Thus, in our cost model, we assume $\cl/\ch = 0.5$. Then~\eqref{eq:mp_cost}
yields
\begin{equation}
\label{eq:mp_cost-specialized}
\cm = (0.5+\rho) \ch,
\end{equation}
where $\rho \in [0,1]$ is the fraction of inner products that must be recomputed in fp16. 
Based on~\eqref{eq:mp_cost-specialized} we can expect that if $\rho < 0.5$, the cost
of the mixed precision fp8/fp16 method will be lower than that of the uniform fp16 one. For each network type and each precision configuration variant,
we perform inference on 10,000 different test inputs and report the resulting test accuracy (that is, the percentage of inputs correctly classified).

The results are presented in Figure~\ref{fig:relu} for ReLU activation functions and in Figure~\ref{fig:tanh} for $\tanh$.  In each figure, the top, middle, and bottom
plots correspond to networks with 3, 5, and 8 layers, respectively. The left and right plots
correspond to the MNIST and Fashion MNIST datasets, respectively.
Each individual plot shows the test accuracy on the $x$-axis and the fraction $\rho$ of inner products (re)computed in fp16 on the $y$-axis,
for each of the three precision configurations: uniform fp8 (a single triangle marker, always found at $y=0$), uniform fp16 (a single star marker, always found at $y=1$),
and the mixed precision Algorithm~\ref{alg:ALG1} with various values for the tolerance $\tau$ (blue line).  

The figures show that different precision configurations achieve different cost--accuracy tradeoffs.
Without surprise, the uniform fp16 variant is always more accurate than the fp8 one. 
As for the mixed precision variant, we see that decreasing the tolerance $\tau$ increases the accuracy
but also increases the fraction of inner products that need to be recomputed in fp16.
Based on the cost model~\eqref{eq:mp_cost-specialized}, we also plot a dashed
line at $\rho=0.5$, the maximum value for which the cost of the mixed precision
algorithm remains less than the uniform fp16 one. Hence blue points below that
dashed line are potentially of interest.

\paragraph{\textbf{Discussion of the results when using ReLU activation functions.}}

The results of these experiments are reported in Figure~\ref{fig:relu}.
For the ReLU function for any choice of the tolerance $\tau$, only a tiny
fraction of the inner products need to be recomputed in fp16. This is due to the fact that
$\kappa_{\phi}(x) = 0$ if $x < 0$, meaning that any inner product whose result is negative will
systematically be kept in low precision regardless of $\tau$.
As it turns out, the percentage of negative inner products, and thus of zero condition numbers,
is extremely large. Table~\ref{tab:kappa-phi} summarizes these percentages (averaged over all inputs)
for the different types of networks;
they are very large regardless of the dataset or of the number of layers,
exceeding 75\% in all cases.
This explains why, in the left plots, the blue points corresponding to the mixed precision configuration
never exceed a fraction of $\rho=0.25$ inner products recomputed in fp16. Thus, we are far below
the $\rho=0.5$ limit and we can expect the mixed precision variant to be significantly faster
than the uniform fp16 one.

\begin{figure}[htbp]
  \centering
  \subfloat[3 layers]{%
    \includegraphics[width=0.48\linewidth]{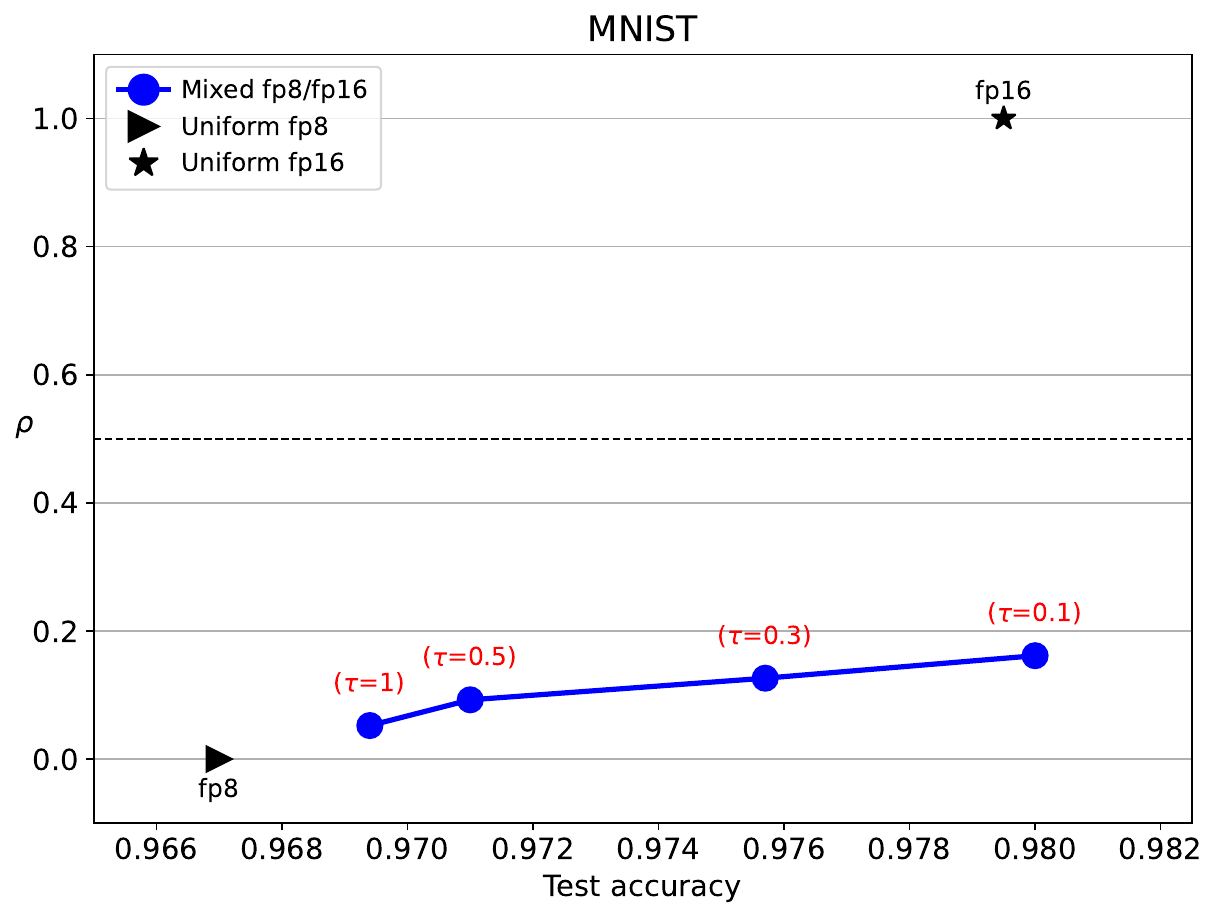}%
    \includegraphics[width=0.48\linewidth]{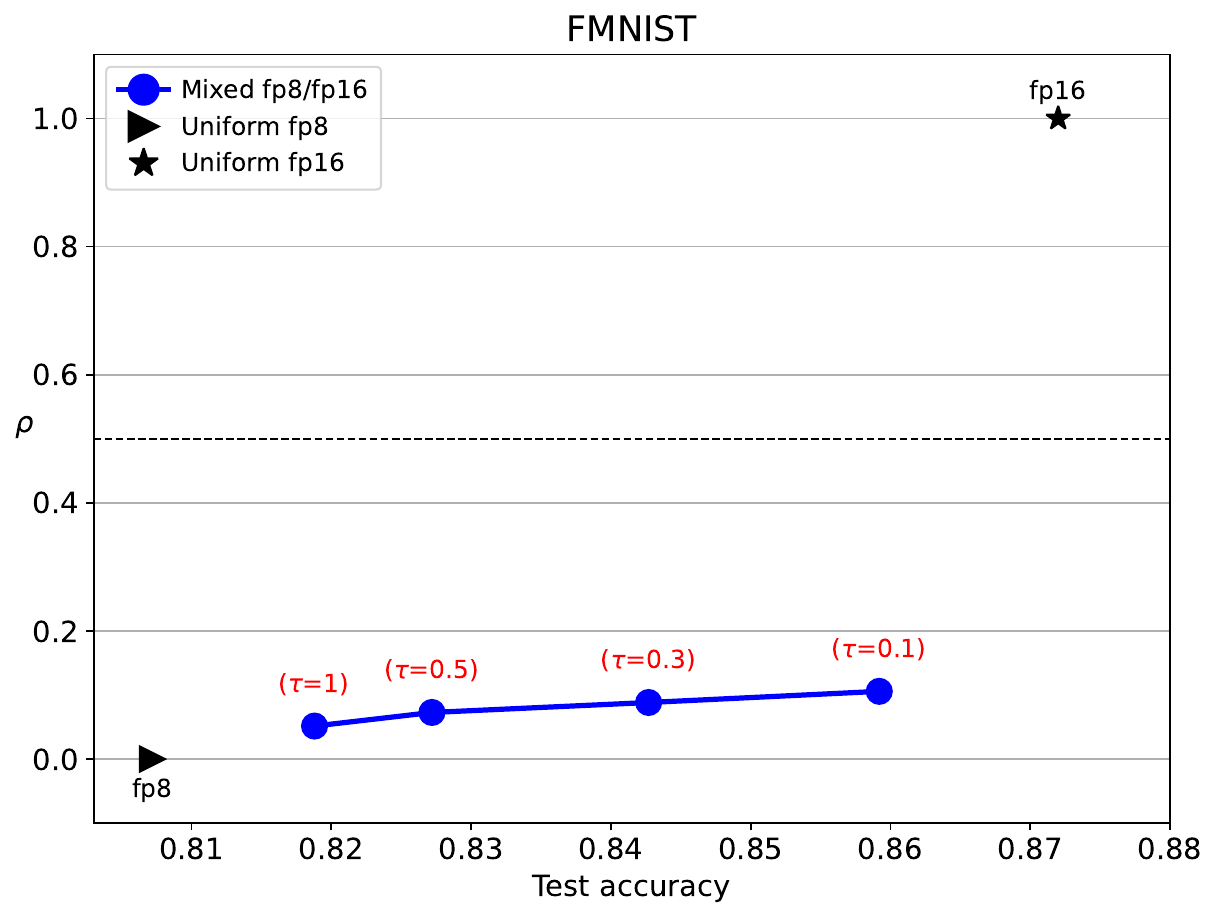}%
  }\hfill
  \subfloat[5 layers]{%
    \includegraphics[width=0.48\linewidth]{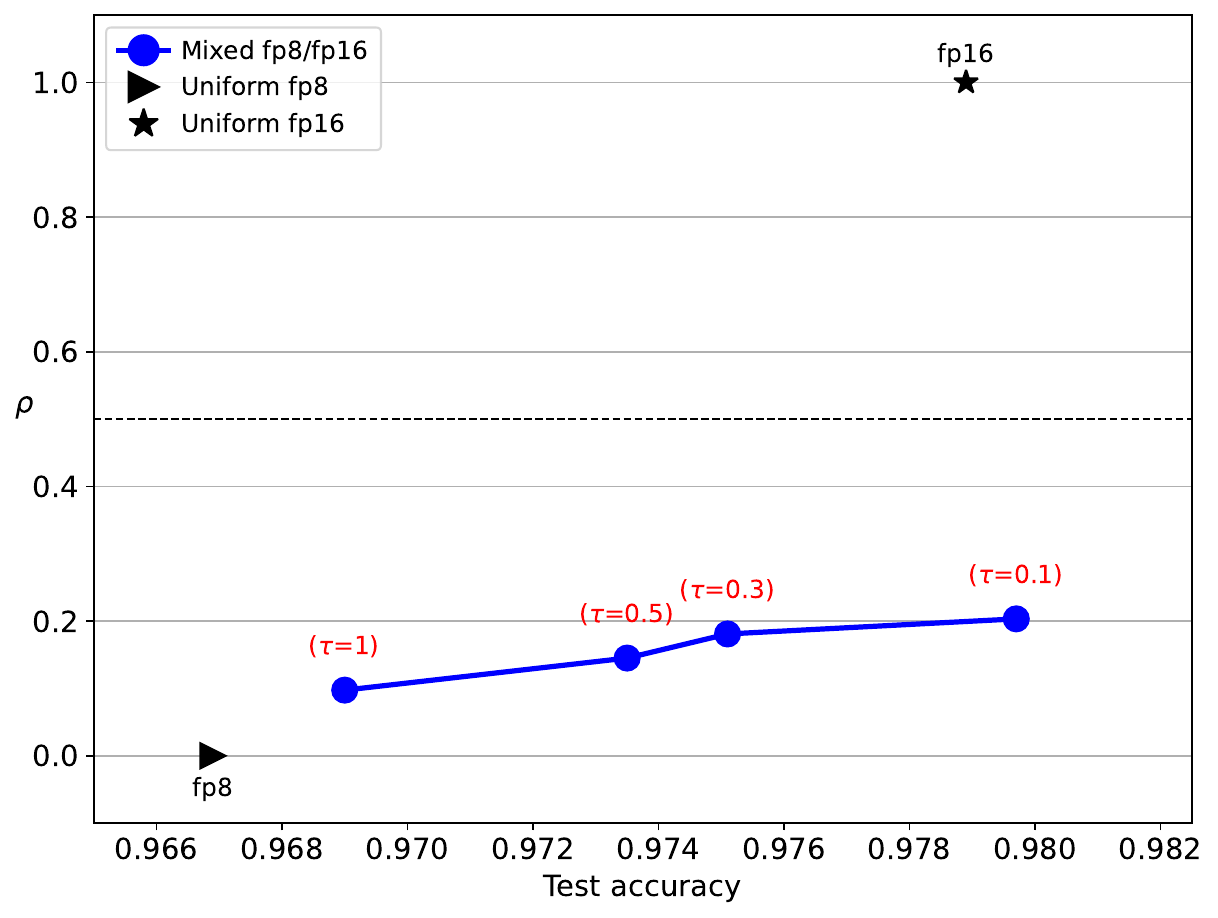}%
    \includegraphics[width=0.48\linewidth]{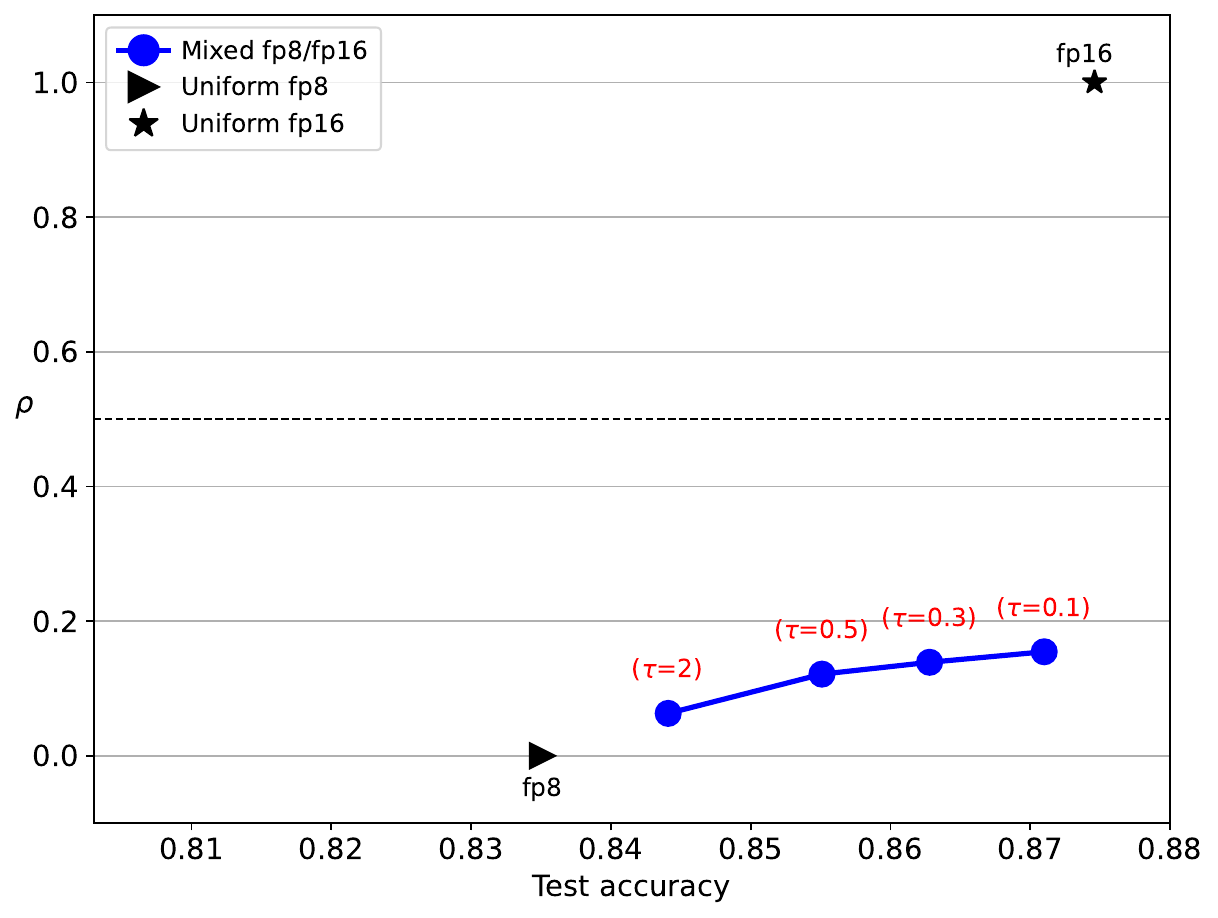}%
  }\hfill
  \subfloat[8 layers]{%
    \includegraphics[width=0.48\linewidth]{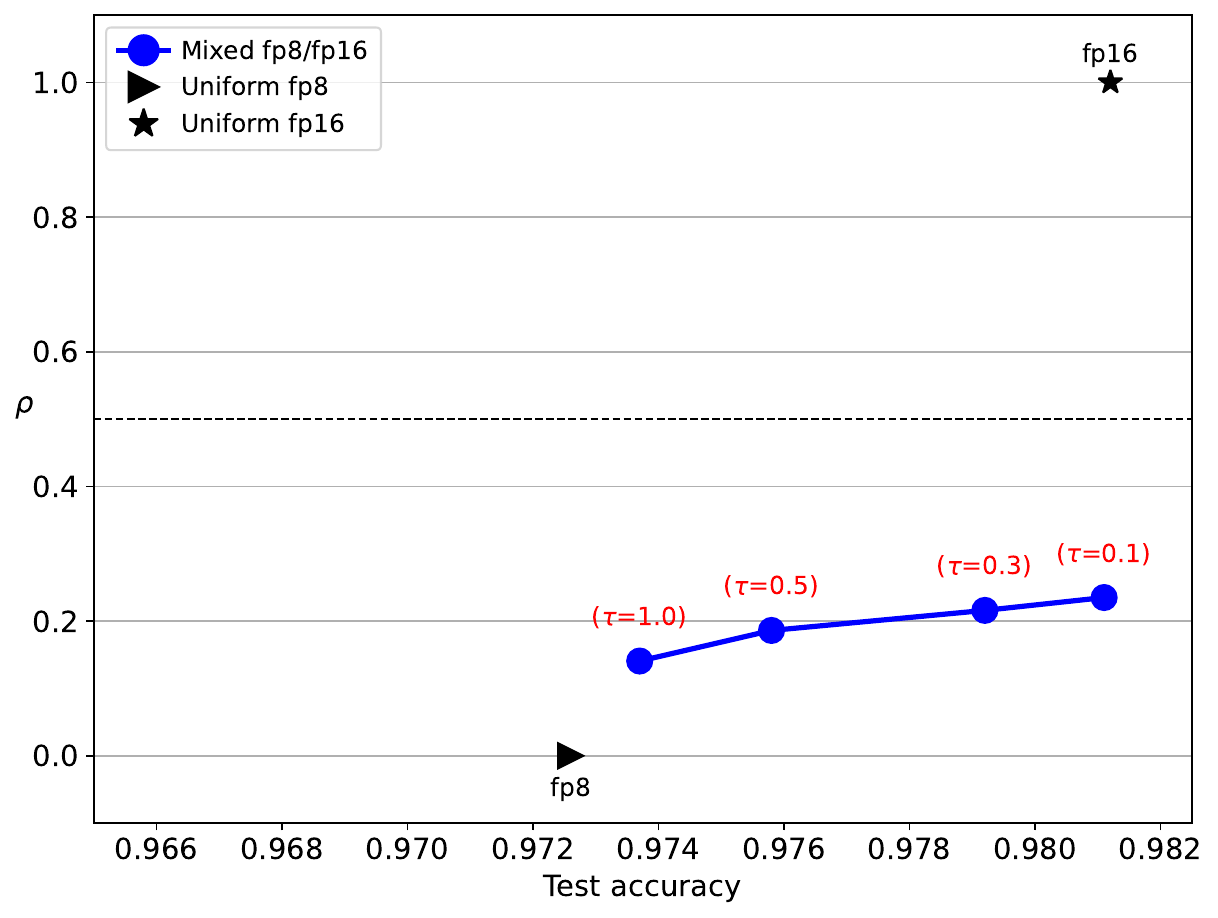}%
    \includegraphics[width=0.48\linewidth]{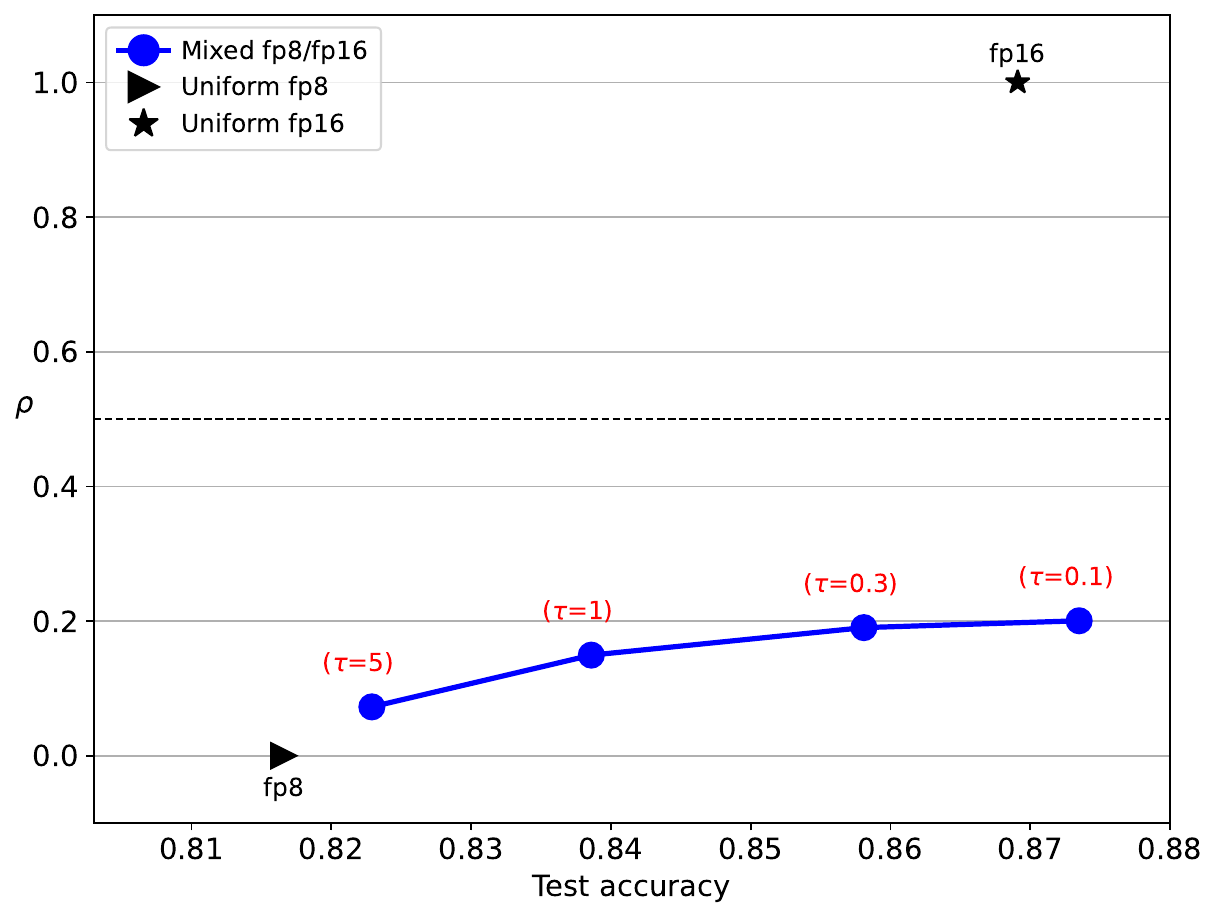}%
  }
\caption{Cost--accuracy tradeoff achieved by different precision configurations on the MNIST (left) and Fashion MNIST (right) datasets, for multilayer perceptron
  networks with 3 (top), 5 (middle), or 8 (bottom) layers, 
	using ReLU  activations. The $x$-axis plots the test accuracy of the inference on the $10,000$ samples of the dataset; the $y$-axis
	plots the fraction $\rho$ of inner products (re)computed in fp16. For the mixed precision configuration (Algorithm~\ref{alg:ALG1}), each point corresponds to
	a different value of the tolerance $\tau$ as indicated.}
  \label{fig:relu}
\end{figure}

\begin{table}
\centering
\caption{Average percentage of zero values in the condition number of ReLU
activations for multilayer perceptron networks with 3, 5, or 8 layers trained
on the MNIST and Fashion MNIST datasets, using fp8 arithmetic.\label{tab:kappa-phi}
}

\begin{tabular}{ccc}
	\toprule Multilayer Perceptron Configuration
	& MNIST &  Fashion MNIST \\
	\midrule
	$3$ layers & $84\%$ & $90\%$  \\
	$5$ layers & $80\%$ & $85\%$  \\
	$8$ layers & $77\%$ & $80\%$  \\
	\bottomrule
\end{tabular}
\end{table}

Despite the large number of operations performed in fp8,  the mixed precision
variant always achieves a better accuracy than the uniform fp8 variant. More importantly,
for a sufficiently small tolerance $\tau$, its accuracy matches that of the uniform fp16 variant.
Thus, the mixed precision variant is \textit{faster yet equally as accurate} as the uniform fp16 variant.
It is interesting to note that as we increase $\tau$,
the fraction $\rho$ of inner products needing to be recomputed in fp16 does slightly decrease, from roughly 0.2 to 0.1.
Since for ReLU $\kappa_{\phi_\ell}$ is either 0 or 1, this behavior is explained by the variations in the components
of $\kappa_{v_\ell}$. Specifically, for very small values of $\tau$, components with $\kappa_{\phi_\ell}=1$ will always be recomputed in fp16.
As we increase $\tau$, some of these components may be kept in fp8 if $\kappa_{v_\ell}$ is small enough, further reducing the fraction
of fp16 computations. However, the figure shows that the test accuracy quickly degrades when doing so, for a cost reduction that is not that
significant. Therefore, these experiments suggest that for ReLU activations, a good rule of thumb is to recompute in fp16 all positive
inner products (for which $\kappa_{\phi_\ell}=1$).

All these observations hold consistently for
all the tested networks, even as we increase the number of layers, both for the MNIST and Fashion MNIST datasets.

\paragraph{\textbf{Discussion of the results when using tanh activation functions.}}

For the $\tanh$ function, the situation is quite different, as shown in Figure~\ref{fig:tanh}.
The fraction $\rho$ of inner products needing to be recomputed in fp16 quickly increases as $\tau$ decreases, so that not all
mixed precision configurations are interesting.
Indeed, in view of \eqref{eq:mp_cost-specialized}, all choices of $\tau$ that demand to recompute more
than $\rho=0.5$ of the inner products
(blue points above the dashed line) should be discarded, since they
are more expensive than the uniform fp16 variant, and yet achieve a lower test accuracy as shown in the figure.
However, some choices of $\tau$ still provide an
interesting compromise between accuracy and cost. The largest values of $\tau$ (for example, $\tau=5$)
often still allow for a slight improvement of the accuracy with respect to uniform fp8, 
which comes almost for free since $\rho\approx 0.9$ in these cases. Alternatively,
more intermediate values of $\tau$ (for example, $\tau=1$) can achieve much more significant accuracy improvements (without, however, reaching the
same accuracy as fp16), for a cost that is in between that of the uniform fp8 and fp16 variants (for example, 
$\rho\approx0.3$, which corresponds to a $20\%$ cost reduction with respect to uniform fp16 in view of \eqref{eq:mp_cost-specialized}).

Overall, these experimental results support the conclusions of our analysis, confirming that it is indeed meaningful to compute different components
of the layers in different precisions, and highlight the potential of the proposed Algorithm~\ref{alg:ALG1} 
to improve the cost--accuracy tradeoff, particularly in the case of ReLU activations. 

\begin{figure}[htbp]
  \centering
  \subfloat[3 layers]{%
    \includegraphics[width=0.48\linewidth]{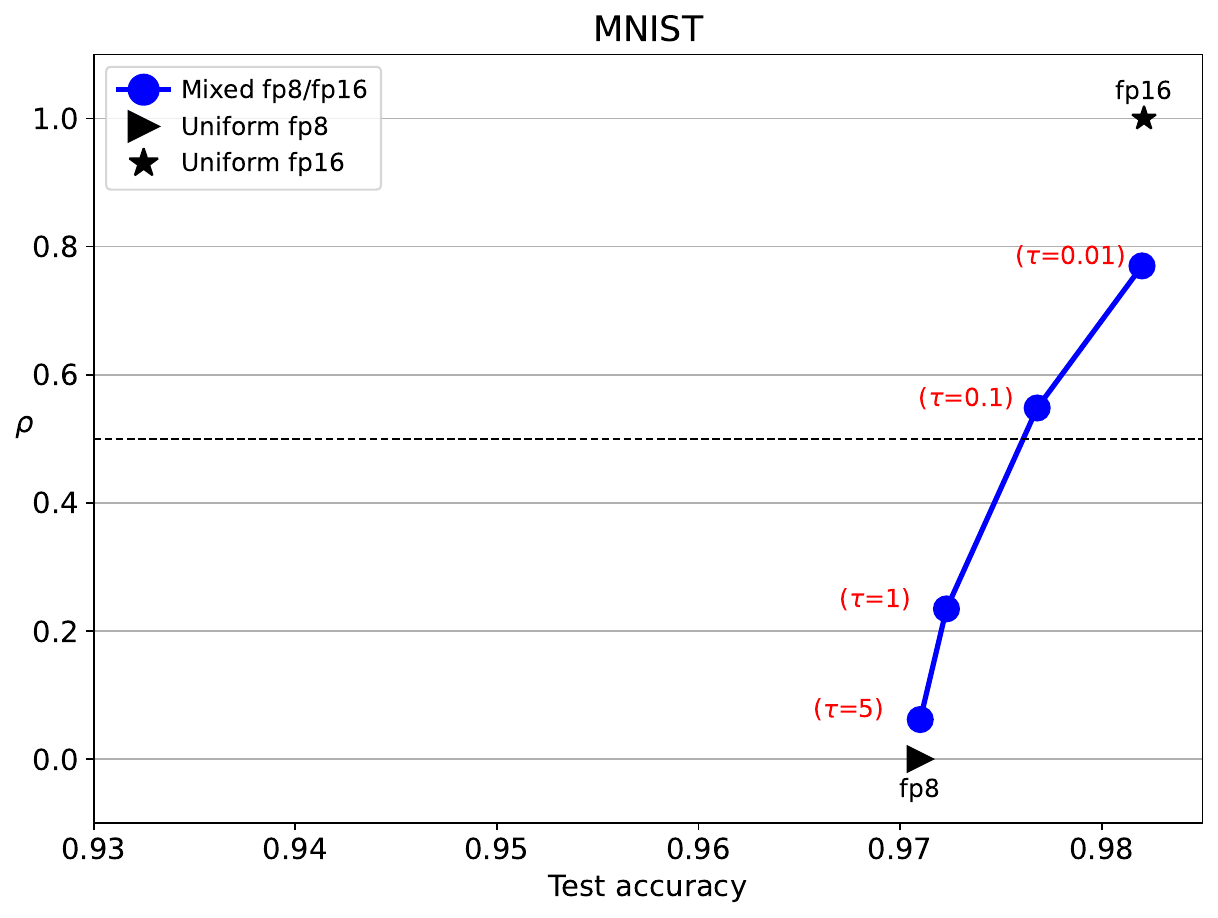}%
    \includegraphics[width=0.48\linewidth]{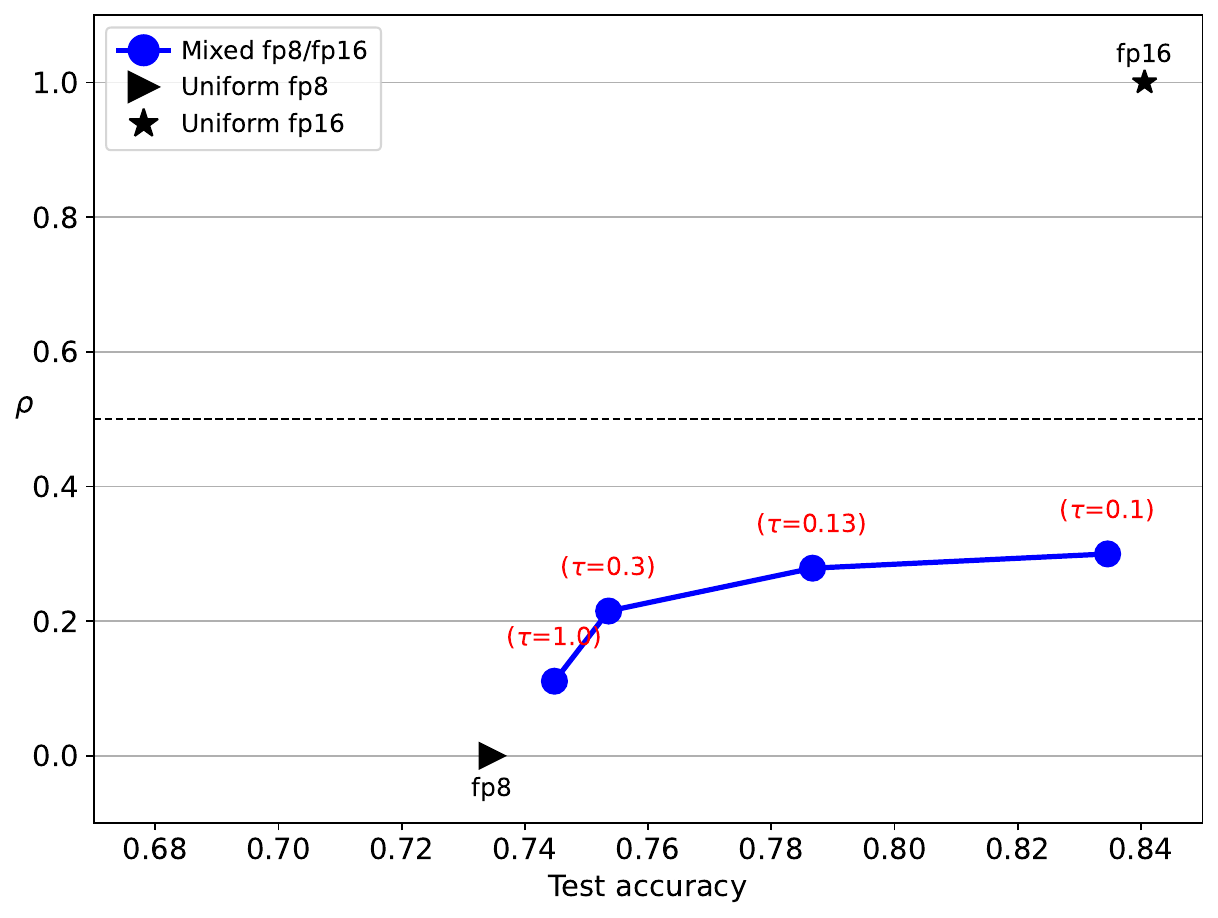}%
  }\hfill
  \subfloat[5 layers]{%
    \includegraphics[width=0.48\linewidth]{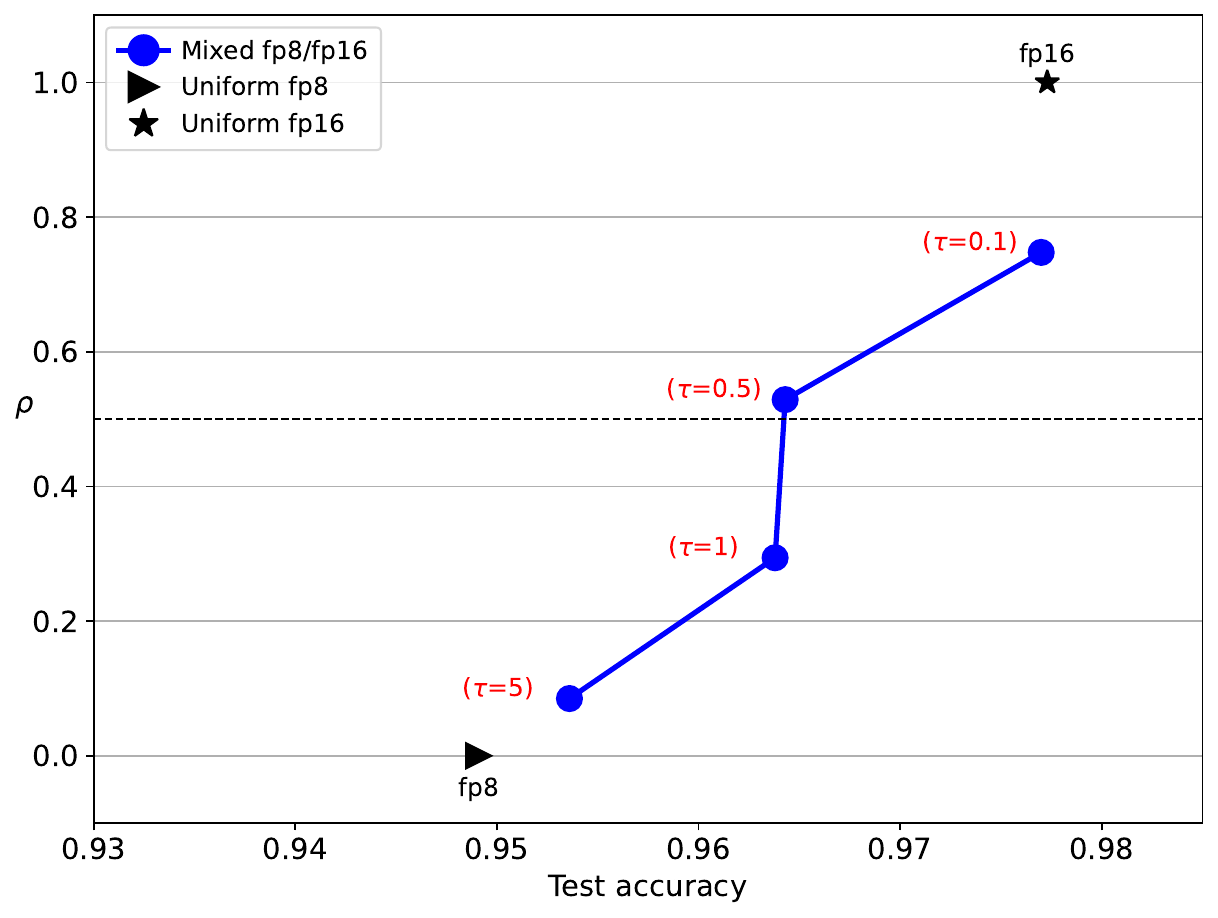}%
    \includegraphics[width=0.48\linewidth]{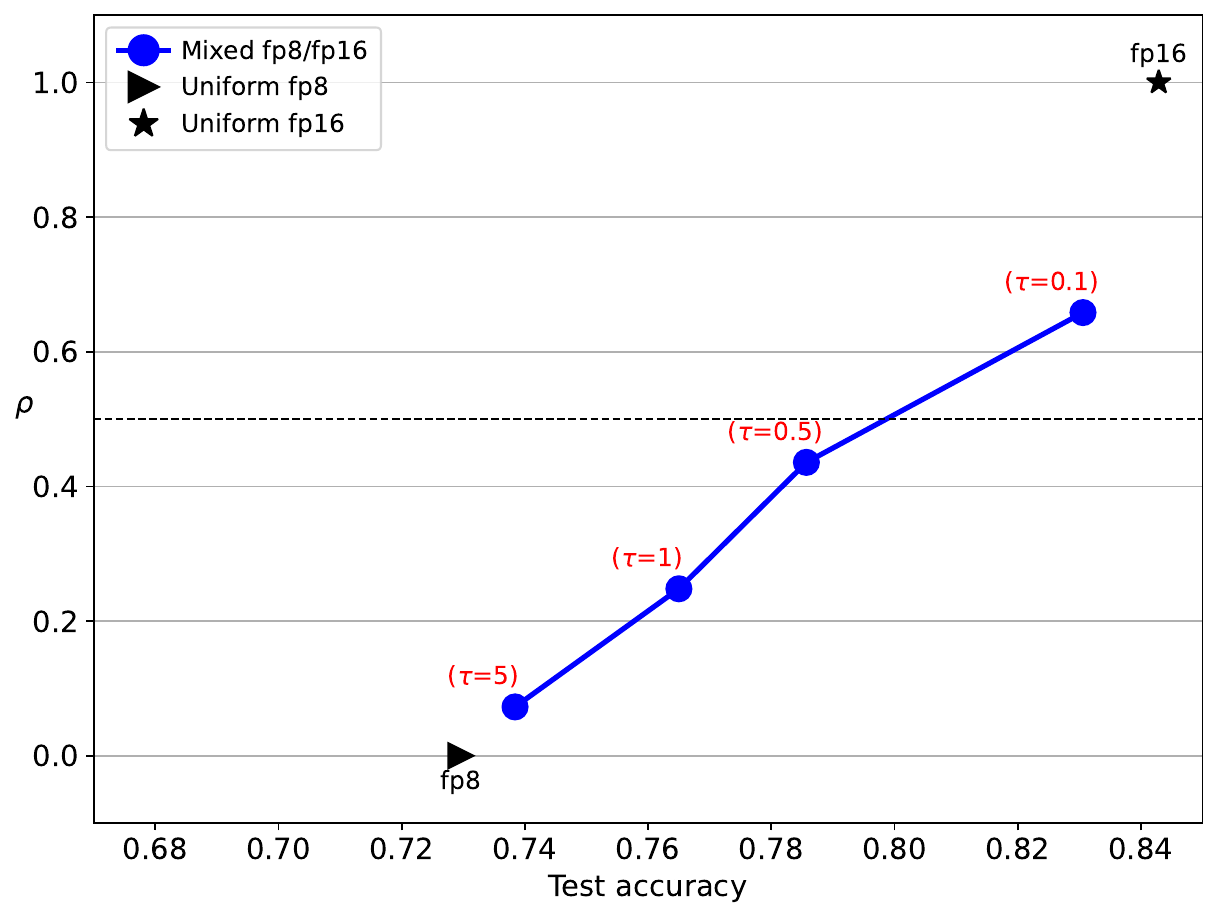}%
  }\hfill
  \subfloat[8 layers]{%
    \includegraphics[width=0.48\linewidth]{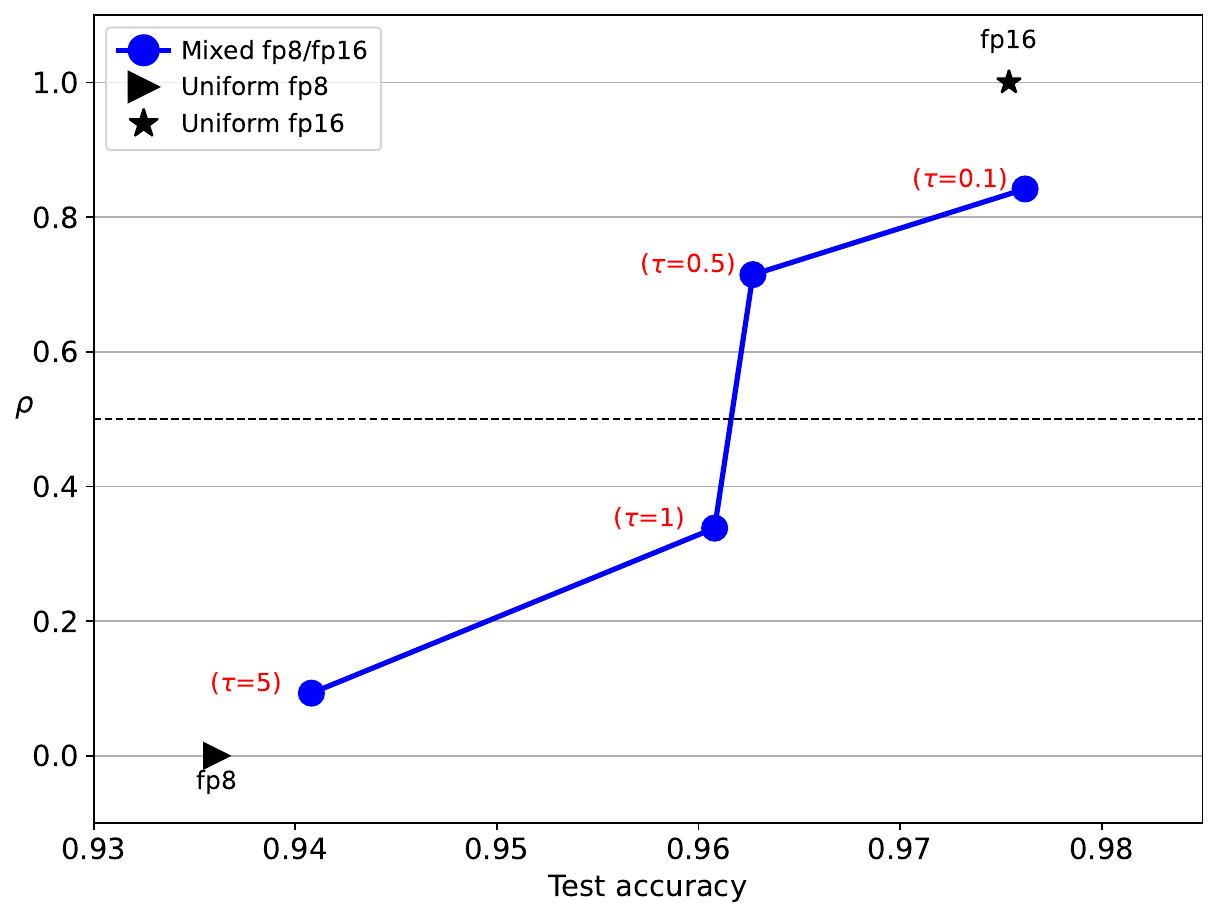}%
    \includegraphics[width=0.48\linewidth]{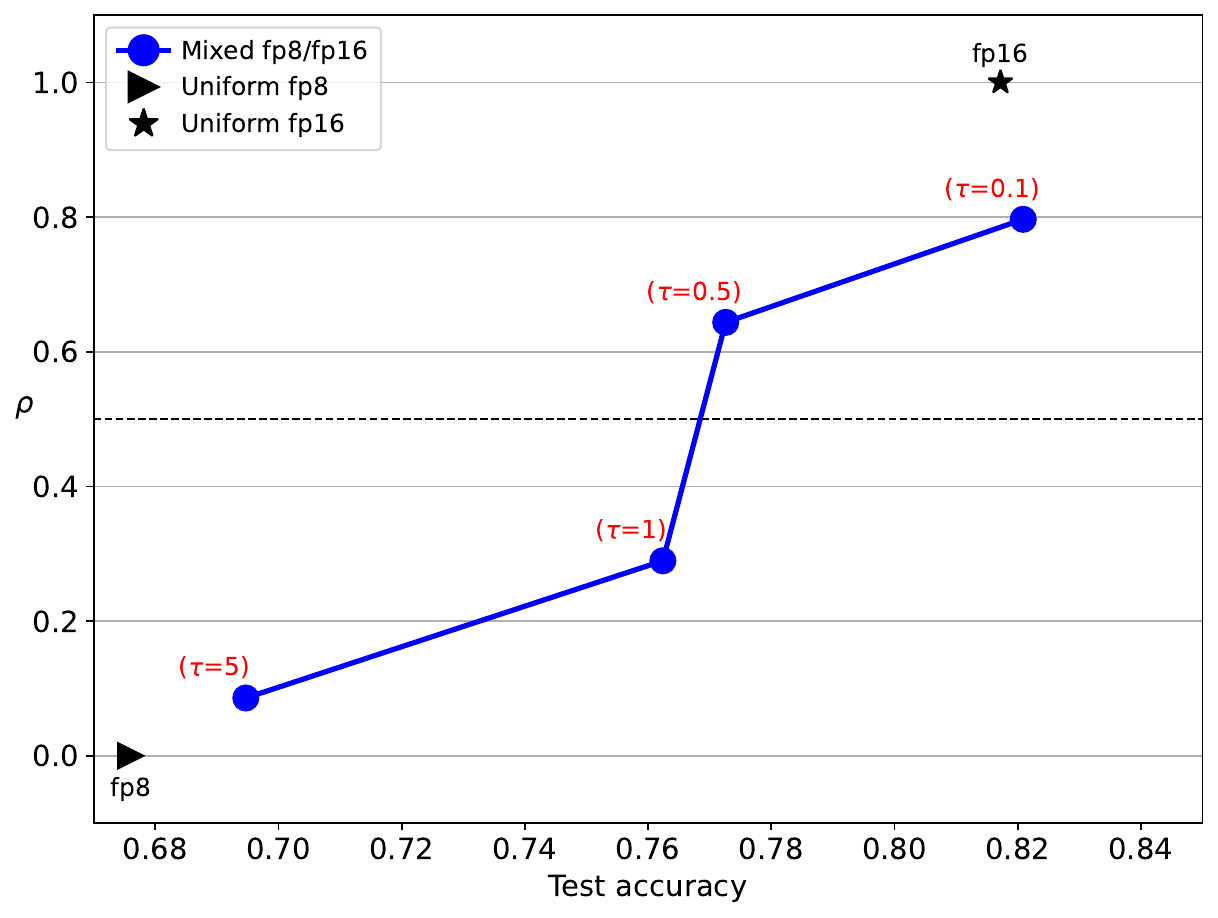}%
  }
  \caption{Same as Figure~\ref{fig:relu} but with $\tanh$ activations.}
	\label{fig:tanh}
\end{figure}

\section{Conclusion: discussion, limitations, and perspectives}
\label{sec:concl}

We have considered the problem of using mixed precision accumulation in the
matrix--multiply accumulate operations for neural network inference. 
In order to do so, we investigated the propagation of errors in the inference,
based on a generic error model that applies in particular to floating-point arithmetic. 
Specifically, we have carried out a componentwise forward error analysis,  
whose main conclusion is reported in Theorem~\ref{thm.main}. This key result shows that
the errors incurred in each inner product of each layer are proportional to
the condition number of the inner product and to the condition number of the activation functions. 
Therefore our analysis suggests (see Corollary~\ref{corollary}) to choose the
precision of each inner product to be inversely proportional to this product of
condition numbers. 

We have leveraged this insight by developing an inference algorithm with mixed precision accumulation.
We introduced some approximations in order to cheaply estimate the condition numbers, leading to the practical approach 
outlined in Algorithm~\ref{alg:ALG1} and illustrated in Figure~\ref{fig.schema}. 
We have experimentally validated the soundness and potential of this approach
on multilayer perceptrons networks with
ReLU and $\tanh$ activations. Our experimental results 
show indeed that the proposed mixed precision approach 
can significantly improve the cost--accuracy tradeoff:
in most cases, it is more accurate than the low precision baseline (fp8 in our tests) 
and less expensive than the high precision baseline (fp16 in our tests).

Despite these promising preliminary results, our approach does present some limitations that could be the object of future work. 
\begin{itemize}
\item 
 Like most traditional rounding error analyses, our analysis assumes the
 absence of overflow and underflow, that is, focuses on the effect of rounding
 errors only. However, this assumption may not be satisfied for some deep
 learning formats with very narrow range. Fortunately, the recent analysis of
 Mary and Mikaitis~\cite{mami25} for matrix multiplication provides some
 reassuring guarantees: provided that suitable scaling is used, overflow can be
 prevented and the effect of underflow can be made smaller than that of
 rounding errors.

\item
It is worth recalling that Algorithm~\ref{alg:ALG1} is the result of several simplifications described 
in Section~\ref{sec:simplif}. While we have experimentally observed these simplifications to be reasonable, 
this is only an empirical statement.

\item
Another limitation of Algorithm~\ref{alg:ALG1} is that its effectiveness
depends on the choice of the parameter $\tau$: a too large value may lead to results of quality no better
than the low precision baseline, while a too small value may lead to a
higher cost than the high precision baseline.
That being said, the choice of
$\tau$ seems to be mostly dependent on the choice of the activation function,
rather than on the dataset or on the neural network architecture. For example, with ReLU activations (Figure~\ref{fig:relu}),
$\tau=0.1$ provides satisfactory results across all experiments.

\item
An important question left for future work is the high performance implementation
of Algorithm~\ref{alg:ALG1}. The aim of our experiments has merely been to validate the principle of
the algorithm and its theoretical potential under a certain cost model; our experiments use simulated
lower precisions and so cannot assess to what extent this potential can be effectively translated in practice.
That being said, we believe there is significant hope for the algorithm to be efficient, since its use of mixed precision is focused
on the linear algebra part of the computation, for which efficient mixed precision implementations have been widely achieved~\cite{hima22}.
\end{itemize}

Finally, this work could be extended in a number of ways, giving rise to several interesting perspectives:
\begin{itemize}
\item
The analysis is general enough to cover various network architectures. 
We have focused our experiments on the multilayer perceptron one, but the approach could be adapted to convolutional networks. 
However, an analysis taking into account the specific structure of such networks should lead to sharper bounds.

\item The extension to transformers~\cite{vspu17} is also natural. A first
error analysis has been carried out by Budzinskiy et al.~\cite{bfzp25}, but
further investigation is required, in particular, concerning mixed precision
opportunities. Indeed, transformers involve the softmax function, which
amplifies variations in magnitude in the data---the mixed precision approach
developed here might thus be applicable.

\item
Algorithm~\ref{alg:ALG1} is dynamic in the sense that we first compute all inner products
in low precision and then dynamically decide at runtime which ones to recompute in high precision based on their estimated condition number. While this approach is robust, the recompute overhead cost
makes it practical only if a small fraction of inner products needs to be (re)computed in high precision. 
This overhead cost could be removed by instead statically deciding which inner products to compute
a priori in high precision---regardless of the input. Naturally, this requires ``learning'' the precision configuration on a representative set of input vectors.
This static strategy may achieve better performance--accuracy tradeoffs
than the dynamic strategy proposed here and may be easier to implement efficiently. Moreover, using a static precision
configuration would allow for exploiting mixed precision not just for the accumulation, but for the quantization too.

\item Finally, another meaningful direction for future research would be the
extension of our approach for online mixed-precision training.
\end{itemize}

\section*{Acknowledgments}

This work was partially supported by the InterFLOP (ANR-20-CE46-0009), 
MixHPC (ANR-23-CE46-0005-01), NumPEx ExaMA (ANR-22-EXNU-0002), MEPHISTO (ANR-24-CE23-7039-01), FPT-4 (ANR-24-CE46-7572), and HOLIGRAIL (ANR-23-PEIA-0010) projects of the French National Agency for Research (ANR).

\newpage

\bibliographystyle{myplain2-doi}
\bibliography{strings,tmary,refs}


\end{document}